\documentclass{article}

\usepackage{microtype}
\usepackage{graphicx}
\usepackage{subfigure}
\usepackage{booktabs} 

\usepackage{hyperref}

\usepackage[accepted]{icml2020}
\usepackage{amsmath,amssymb}
\usepackage{algorithm}
\usepackage{algpseudocode}
\usepackage{amsthm}
\usepackage{bbm}
\usepackage{comment}

\newcommand{\R}{\mathbb{R}} 
\newcommand{\E}{\mathbb{E}} 
\DeclareMathOperator*{\argmin}{arg\,min} 

\DeclareMathOperator*{\Diag}{Diag} 
\DeclareMathOperator{\prox}{prox} 
\DeclareMathOperator*{\vect}{vec} 

\newcommand{\normi}[1]{\|{#1}\|}

\newtheorem{theorem}{Theorem} 
\newtheorem*{theorem*}{Theorem}
\newtheorem*{proposition*}{Proposition}
\newtheorem{corollary}{Corollary}

\newtheorem{lemma}{Lemma}
\newtheorem{definition}{Definition}

\newtheorem{proposition}{Proposition}

\newtheorem{remark}{Remark}

\def\argmin{\displaystyle\mathop {\mbox{argmin}}}

\newcommand{\RNum}[1]{\uppercase\expandafter{\romannumeral #1\relax}}

\def\int{\displaystyle\mathop {\mbox{\rm int}}}    

\def\dom{\displaystyle\mathop {\mbox{\rm dom}}}    

\def\sup{\displaystyle\mathop {\mbox{\rm sup}}}

\def\inf{\displaystyle\mathop {\mbox{\rm inf}}}

\def\argmin{\displaystyle\mathop {\mbox{\rm argmin}}}

\def\real{\mathbb R}

\newcommand{\bu}{{\bf u}}

\newcommand{\sign}{\mbox{sign}}

\icmltitlerunning{Efficient Proximal Mapping of the 1-path-norm of Shallow Networks}

\begin{document}

\twocolumn[
\icmltitle{
Efficient Proximal Mapping of the 1-path-norm of Shallow Networks
}



\icmlsetsymbol{equal}{*}

\begin{icmlauthorlist}
\icmlauthor{Fabian Latorre}{equal,lions}
\icmlauthor{Paul Rolland}{equal,lions}
\icmlauthor{Nadav Hallak}{equal,lions}
\icmlauthor{Volkan Cevher}{lions}
\end{icmlauthorlist}

\icmlaffiliation{lions}{Learning, information and optimization systems
laboratory (LIONS), EPFL, Switzerland}

\icmlcorrespondingauthor{Fabian Latorre}{fabian.latorre@epfl.ch}

\icmlkeywords{Machine Learning, ICML}

\vskip 0.3in
]



\printAffiliationsAndNotice{\icmlEqualContribution} 

\begin{abstract}
We demonstrate two new important properties of the 1-path-norm of shallow
    neural networks. First, despite its non-smoothness and non-convexity it
    allows a closed form proximal operator which can be efficiently computed,
    allowing the use of stochastic proximal-gradient-type methods for
    regularized empirical risk minimization. Second, when the activation
    functions is differentiable, it provides an upper bound on the Lipschitz
    constant of the network. Such bound is tighter than the trivial layer-wise
    product of Lipschitz constants, motivating its use for training networks
    robust to adversarial perturbations. In practical experiments we illustrate
    the advantages of using the proximal mapping and we compare the
    robustness-accuracy trade-off induced by the 1-path-norm, L1-norm and
    layer-wise constraints on the Lipschitz constant (Parseval networks).
\end{abstract}

\section{Introduction}
\label{sec:introduction}
Neural networks are the backbone of contemporary applications in machine
learning and related fields, having huge influence and significance both in
theory and practice. Among the most important and desirable attributes of a
trained network are robustness and sparsity.  Robustness, is often defined as
stability to adversarial perturbations, such as in supervised classification
methods. The apparent brittleness of neural networks to adversarial attacks in
this context has been considered in the literature for some time, see e.g.,
\citep{Biggio2013,szegedy2013intriguing,Madry2018} and references therein.

A fundamental question in this regard is how to measure robustness, or more
importantly, how to encourage it.  One prominent approach supported by theory
and practice \citep{raghunathan2018certified,cisse17a}, is to use the Lipschitz
constant of the network function to quantize robustness, and regularization to
encourage it. 

This approach is also supported theoretically with generalization bounds in
terms of the layer-wise product of spectral norms
\citep{Bartlett2017,Miyato2018}, which particularly upper-bounds the Lipschitz
constant. However, a recent empirical study \citep{jiang2020fantastic} has
found in practice a negative \textit{correlation} of this measure with
generalization. This casts doubts on its usefulness and signals the fact that
it is a rather loose upper bound for the Lipschitz constant
\citep{Latorre2020}.

Current methods that compute upper bounds on the Lipschitz constant of neural
networks can be roughly classified into two classes: (i) the class of
\textit{product bounds}, comprising all upper bounds obtained by the
multiplication of layer-wise matrix norms; and, (ii) the class of
convex-optimization-based bounds, which addresses the network as a whole entity
\citep{raghunathan2018certified,Fazlyab2019,Latorre2020}.

A trade-off between computational complexity and quality of the upper bound
seems apparent. An ideal bound would achieve a balance between both properties:
it should provide a good estimate of the constant while being fast and easy to
minimize with iterative first-order algorithms.

Recently, the \textit{path-norm} of the network \citep{Neyshabur2015}
has emerged as a complexity measure that is highly-correlated with generalization
\citep{jiang2020fantastic}. Thus, its use as a regularizer holds an increasing
interest for researchers in the field.

Despite existing generalization bounds \citep{Neyshabur2015}, our understanding
of the optimization aspects of the path-norm-regularized objective is lacking.
\citet{jiang2020fantastic} refrained from using automatic-differentiation
methods in this case because, as they argue, the optimization could fail, thus
providing no conclusion about its qualities.

It is then natural to ask: \textit{how do we properly optimize the
path-norm-regularized objective with theoretical guarantees? What conclusions
can we draw about the robustness and sparsity of path-norm-regularized
networks?} We focus on the \textit{1-path-norm} and provide partial answers to
those questions, further advancing our understanding of this measure. Let us
summarize our main contributions:

\textbf{Optimization. }We show a striking property of the
1-path-norm, that makes it a strong candidate for explicit regularization:
despite its non-convexity, it admits an efficient \textit{proximal mapping}
(Algorithm \ref{alg:prox_binar_multi}). This allows the use of proximal-gradient
type methods which are, as of now, the only first-order optimization algorithms
to provide guarantees of convergence for composite non-smooth and non-convex
problems \citep{BST13}.

Indeed, automatic differentiation modules of popular deep learning frameworks
like PyTorch \citep{pytorch} or TensorFlow \citep{tensorflow} may not compute 
the correct gradient for compositions of non-smooth functions, at points where
these are differentiable \citep{Kakade2019, bolte2019conservative}. Our
proposed optimization algorithm avoids such issue altogether by using
differentiable activation functions like ELU \citep{Clevert2015} and our novel
proximal mapping of the 1-path-norm.

\textbf{Upper bounds. }We show that the \textit{1-path-norm}
\citep{Neyshabur2015} achieves a sweet spot in the computation-quality
trade-off observed among upper bounds of the Lipschitz constant:
it has a simple closed formula in terms of the weights of the
network, and it provides an upper bound on the $(\ell_\infty,
\ell_1)$-Lipschitz constant (cf., \autoref{prop:upper_bound}), which is always
better than the product bound.

\textbf{Sparsity. }Neural network regularization schemes promoting sparsity in
a principled way are of great interest in the growing field of
\textit{compression} in Deep Learning \citep{Han2016,cheng2017survey}.

Our analysis provides a formula (cf. Lemma \autoref{lem:4b}) for
choosing the \textit{strength} of the regularization, which enforces a desired bound on the 
sparsity level  of  the iterates generated by the proximal gradient method. 
This is a suprising, yet intuitive, result, as the sparsity-inducing properties of
non-smooth regularizers have been observed before in convex
optimization and signal processing literature, see e.g., \citep{Bach2012,EK12}.

\textbf{Experiments. }In \autoref{sec:experiments}, we present numerical
evidence that our approach (i) converges faster and to lower values of the
objective function, compared to plain SGD; (ii) generates sparse iterates; and,
(iii) the magnitude of the regularization parameter of the 1-path-norm allows
a better accuracy-robustness trade-off than the common $\ell_1$ regularization
or constraints on layer-wise matrix norms.

\section{Problem Setup}
We consider the so-called shallow neural networks with $n$ hidden neurons and
$p$ outputs $h:\real^m \rightarrow \real^p$ given by
\begin{equation}
\label{eq:nn}
h_{V,W}(x) = V^T \sigma(Wx),
\end{equation}
where $V \in \R^{n \times p}, W\in \R^{n\times m}$ and $\sigma:\real
\rightarrow \real$ is some differentiable activation function with derivative
globally bounded between zero and one. This condition is satisfied, for
example, by the ELU or softplus activation functions. 
To control the robustness
of the network to perturbations of its input $x$, we want to regularize
training using its Lipschitz constant as a function of the weights $V$ and $W$.

To properly define this constant, we utilize the $\ell_\infty$-norm
for the input space, and the $\ell_1$-norm for the output space.
Exact computation of such constant is a hard task. A simple and easily
computable upper bound can be derived by the product of the layer-wise Lipschitz
constants, however, it can be quite loose.

We derive an improved upper bound which is still easy to compute. In the
following, we denote with $\|W\|_\infty$ the operator norm of a matrix $W$ with
respect to the $\ell_\infty$ norm for both input and output space; it is equal
to the maximum $\ell_1$-norm of its rows. We denote with $\|V\|_{\infty, 1}$
the operator norm of the matrix $V$ with respect to the $\ell_\infty$ norm in
input space and $\ell_1$-norm in output space; it is equal to the sum of the
$\ell_1$ norm of its columns.

\begin{theorem}
\label{prop:upper_bound}
Let $h_{V, W}(x)=V^T \sigma(Wx)$ be a network such that the derivative of the
activation $\sigma$ is globally bounded between zero and one. Choose the
$\ell_\infty$- and $\ell_1$-norm for input and output space, respectively. The
    Lipschitz constant of the network, denoted by $L_{V,W}$ is bounded as follows:
\begin{equation}
    \label{eq:upper_bound}
    L_{V, W} \leq \sum_{i=1}^n\sum_{j=1}^m\sum_{k=1}^p |W_{ij}V_{ik}| \leq \|V^T\|_{\infty, 1} \|W\|_\infty
\end{equation}
\end{theorem}
The proof is provided in appendix \ref{app:upper_bound}. The term in the middle
of inequality \eqref{eq:upper_bound} belongs to the family of
\textit{path-norms}, introduced in \citet[Eq. (7)]{Neyshabur2015}. Throughout,
we refer to it as the \textit{1-path-norm}.

Notice that although the path-norm and layer wise product bounds can be equal,
this only happens in the following worst case: For the weight matrix in the
first layer, the 1-norms of every row are equal. Thus, in practice the bounds
can differ drastically.

\begin{remark}
    In practice, one might want to regularize each ouput of the network in a
    different way according to some weighting scheme
    \citep{raghunathan2018certified}. Precisely, the 1-path-norm of the network
    is equal to the sum (with equal weight) of the 1-path-norm of each output.
    A weighted version of the 1-path-norm can be defined to account for such a
    weighting scheme. All our results can be adapted to this scenerio, with
    minor changes.
\end{remark}

We now turn to the task of minimizing an empirical risk functional regularized by
the improved upper bound on the Lipschitz
constant given in (\ref{eq:upper_bound}): 
\begin{equation}
    \label{eq:objective}
    \min_{V,W} \E_{(x, y)}[\ell(h_{V, W}(x), y)]
    + \lambda \sum_{i=1}^n\sum_{j=1}^m\sum_{k=1}^p |W_{ij}V_{ik}| 
\end{equation}
The objective function in problem \eqref{eq:objective} is composed of an
expectation of a nonconvex smooth loss, and a nonconvex nonsmooth regularizer,
meaning that it is essentially a composite problem (cf. \citep[Ch. 10]{B17}).
That is, the objective function \eqref{eq:objective} can be cast as 
use these notation hereafter)
\begin{equation}
\label{eq:objective_fn}
\min_{V, W} \mathcal{F} (V, W) \equiv f(V, W) + \lambda g(V, W),
\end{equation}
where $f$ is a nonconvex continuously differentiable function, and $g(V, W)=
\sum_{i=1}^n\sum_{j=1}^m\sum_{k=1}^p |W_{ij}V_{ik}|$ is a continuous,
nonconvex, nonsmooth, function. We assume that the objective function is
bounded below, i.e., $\inf  \mathcal{F}  :=\mathcal{F}_*>0$.

A natural choice for a scheme to obtain critical points for
(\ref{eq:objective_fn}) is the proximal-gradient framework.  However, for a
nonconvex $g$, solving the proximal gradient problem is a hard problem in
general.  In Section \ref{sec:compute_prox} we develop a method that computes
the proximal gradient with respect to $g$ efficiently.

To streamline our approach and techniques in a compact and user-friendly
manner, we will illustrate the majority of our results and proofs via the
particular single-output scenario in which $h$ and $g$ are reduced to
\begin{equation*}
h_{v, W}(x) = v^T \sigma(Wx), \ g (v,W) \equiv \lambda \| \vect(\Diag(v)W) \|_1.
\end{equation*}
The multi-output case follows from the same techniques and insights, however,
requires more tedious computations and arguments, on which we elaborate in
Section \ref{sec:multi}, and detail in the appendix.

\section{The Prox-Grad Method}

 Assume that $f$ has a Lipschitz continuous gradient with Lipschitz constant $L>0$,  that is
\begin{equation*}
\| \nabla f(z) - \nabla f(u) \| \leq L \| z - u\|, \qquad \forall z,u \in\real^n.
\end{equation*}
The prox-grad method is described by Algorithm \ref{alg:prox_grad};  since $g$ is nonconvex, the prox in \eqref{eq:proximal} can be a set of solutions.
\begin{algorithm}[H]
	\caption{Prox-Grad Method}
	\label{alg:prox_grad}
	\textbf{Input:} $z^0 \equiv \vect(V^0,W^0) \in \R^{p\cdot n +  n \cdot m}$, $\{\eta^k\}_{k\geq 0}$.
	\begin{algorithmic}[1]
		\For{$k=0,1,\ldots$}
		\State Compute $G^k = \nabla f(z^k)$
		\State $z^{k+1} \leftarrow \mathrm{prox}_{\eta^k g} (z^k - \eta^k G^k)$ 
		\EndFor
	\end{algorithmic}
\end{algorithm}
Theoretical guarantees for the  prox-grad method with respect to a nonconvex regularizer were established by \cite{BST13} (for a more general prox-grad type scheme).  
\begin{theorem}[Convergence guarantees]
	\label{thm:2}
	Let $\{z^k\}_{k\geq 0}$ be a sequence generated by Algorithm \ref{alg:prox_grad} with $\{\eta^k\}_{k\geq 0} \subseteq (0,1/L)$. 
	Then
	\begin{enumerate}
		\item Any accumulation point of $\{z^k\}_{k\geq 0}$ is a critical point of \eqref{eq:objective_fn}.
		\item If $f$ satisfies the Kurdyka-Lojasiewicz (KL) property, then  $\{z^k\}_{k\geq 0}$ converges to a critical point.
		\item Suppose that $\eta_k$ is chosen such that there exists $c > 0$ such that $\sum_{k=0}^{K}\frac{1}{\eta_k} \geq c K$ for any $K\geq 0$.
		Then
		$$ \min_{k=0,\dots,K}  \|z^{k+1} - z^k \|_2 \leq \sqrt{\frac{2(\mathcal{F}(z^0)  - \mathcal{F}_*)}{(c-L)K}}.$$
	\end{enumerate}
\end{theorem}
\begin{proof}
	See Section \ref{app:prox_grad} in the appendix.
\end{proof}
\begin{remark}[On KL related convergence rate]
	A convergence rate result under the  KL property can be derived with respect to the desingularizing function; see \cite{BST13} for additional details.
\end{remark}
\begin{remark} [On the stochastic prox-grad method]
	The literature does not provide any theoretical guarantees for a prox-grad type method that uses stochastic gradients (i.e., replacing $G^k$ with an approximation of $\nabla f(z^k)$) under our setting.
	Recently, \cite{MT19} studied stochastic prox-grad methods, however, their results rely on the assumption that  the regularizer is Lipschitz continuous, which is not satisfied by our robust-sparsity regularizer.
\end{remark}

\section{Computing the Proximal Mapping}
\label{sec:compute_prox}
Throughout this section we assume the single-output setting.  The  path-norm
regularizer  we propose is a nonconvex nonsmooth function, suggesting that the
prox-grad scheme in Algorithm \ref{alg:prox_grad} is intractable.

In this section we will not only prove that in fact \textit{it is tractable} in
the single output case, but that it can also be  \textit{implemented
efficiently} with complexity of  $O(m\log(m))$; we prove the stated in detail
in Section \ref{app:prox_computation1}, and provide here a concise version.

Denote the given pair $(x, Y)$ by $z$. The proximal mapping with respect to
$\lambda g$ at $z$ is defined as
\begin{equation}
    \label{eq:proximal}
    \prox_{\lambda g}(z) = \argmin_{u} \lambda g(u) + \frac{1}{2}
        \normi{\vect(u - z)}_2^2.
\end{equation}
By the choice of $g$, the objective function in \eqref{eq:proximal} is coercive
and lower bounded, implying that there exists an optimal solution (cf.
\citep[Thm. 2.32]{B14}). 
\begin{remark}
    The derivations in this section can be easily adapted and used with
    \textit{adaptive gradient methods} like Adagrad \citep{duchi2011adaptive},
    by a careful handling of the per-coordinate scaling coefficients.
\end{remark}

\begin{lemma} [Well-posedness of \eqref{eq:proximal}]
	\label{lem:12}
	For any $\lambda\geq 0$ and any $(u,z)$, the problem \eqref{eq:proximal} has a global optimal solution.
\end{lemma}
Additionally, we have that  \eqref{eq:proximal} is separable with respect to
the $i$-th entry of the vector $v$ and the $i$-th row of the matrix $W$,
meaning that problem \eqref{eq:proximal} can be solved in a distributed manner
by applying the same solution procedure coordinate-wise for $v$ and row-wise
for $W$.  In light of this, let us consider the $i$-th row related problem
\begin{equation}
\label{eq:7b}
\min_{v, w\in\real\times\real^m} \frac{1}{2} (v-x)^2 + \frac{1}{2} \sum_{j=1}^m (w_j - y_j)^2 +
\lambda |v| \sum_{j=1}^m |w_j|.
\end{equation}
The signs of the elements of the decision variables in (\ref{eq:7b}) are determined by the signs of $(x, y)$, and consequently, the problem in  (\ref{eq:7b})  is equivalent to  
\begin{equation}
\label{eq:cw_prox_prob}
\min_{v, w\in\real_+\times\real_+^m} \frac{1}{2} (v-|x|)^2 + \frac{1}{2} \sum_{j=1}^m (w_j - |y_j|)^2 +
\lambda v \sum_{j=1}^m w_j.
\end{equation}
\begin{lemma}
	\label{lem:5b}
	Let $(v^*, w^*)\in\real_+\times\real_+^n$ be an optimal solution of \eqref{eq:cw_prox_prob}.
	Then $(\sign(x) \cdot v^*, \sign(y) \circ w^*)$ is an optimal solution of problem \eqref{eq:7b}.
\end{lemma}

Denote  
$$ h_\lambda(v,w;x,y)= \frac{1}{2} (v-|x|)^2 + \frac{1}{2} \sum_{j=1}^m (w_j - |y_j|)^2 +
\lambda v \sum_{j=1}^m w_j.$$ 
Although $h_\lambda$ is nonconvex, we will show that a global optimum  to \eqref{eq:cw_prox_prob} can be obtained efficiently  by utilizing several tools, the first being the first-order optimality conditions of \eqref{eq:cw_prox_prob} (cf.  \citep[Ch. 9]{B14}) given below.
\begin{lemma} [Stationarity conditions]
	\label{lem:2b}
	Let $(v^*, w^*)\in \real_+\times\real_+^m$ be an optimal solution of \eqref{eq:cw_prox_prob} for a given $(x,y)\in \real\times\real^m$.
	Then 
	\begin{align*}
	w_j^* &= \max \left\lbrace 0, |y_j| - \lambda v^* \right\rbrace \ \text{ for any }  j=1,2,\ldots,m,\\
	v^* &= \max\left\lbrace 0, |x| -   \lambda \sum_{j=1}^{m} w^*_j \right\rbrace.
	\end{align*}
\end{lemma}
A key insight following Lemma \ref{lem:2b} is that: the elements of any solution to \eqref{eq:cw_prox_prob}, satisfy a monotonic relation in magnitude,  correlated with  the magnitude of the  elements of $y$; this is formulated by the next corollary.
\begin{corollary}
	\label{cor:1b}
	Let $(v^*, w^*)\in \real_+\times\real_+^m$ be an optimal solution of \eqref{eq:cw_prox_prob} for a given $(x,y)\in \real\times\real^m$. 
	Then
	\begin{enumerate}
		\item The vector $w^*$ satisfies that  for any $j,l\in\{1,2,\ldots,m \}$ it holds that $w^*_j \geq w^*_l$ only if $|y_j| \geq |y_l|$.
		\item Let $\bar{y}$ be the sorted vector of $y$ in descending magnitude order. Suppose that $v^* > 0$ and let $s= | \{ j : sw^*_j > 0 \} |$.
		Then,
		\begin{equation}
		\label{eq:2b}
		v^* = \dfrac{1}{1-s \lambda^2}\left(|x| - \lambda\sum_{j=1}^s|\bar{y}_j|\right),
		\end{equation}
		where we use the convention that $\sum_{j=1}^0|\bar{y}_j| = 0$.
	\end{enumerate}
\end{corollary}
\begin{proof}
	The first part follows trivially from the stationarity conditions on $w^*$ given in Lemma \ref{lem:2b}.
	
	From the first part and the conditions in Lemma \ref{lem:2b} we have that $\sum_{j=1}^{m} w^*_j =  \sum_{j=1}^s|\bar{y}_j| - \lambda s v^*$.
	Plugging the latter to  the stationarity condition on $v^*$ (given in Lemma \ref{lem:2b}) then implies the required.
\end{proof}

\begin{remark} \label{rem:sorting}
Corollary~\ref{cor:1b} implies that the solution vector $w^*$ is ordered in the
    same way as $|y|$. Thus, the $s$ non-zero entries of $w^*$ are precisely the ones
    corresponding with the $s$ largest entries of $|y|$. 

Without loss of generality, we assume hereafter that the input  $y$ is already
    sorted in decreasing  order, such that the $s$ non-zero entries of $w^*$
    are always the  first $s$ entries.
\end{remark}
To supplement the results above, we now show that we can actually upper-bound the sparsity level of the prox-grad output by adjusting the value of $\lambda$. 
\begin{lemma}[Sparsity bound]
	\label{lem:4b}
	Let $(v^*, w^*)\in \real_+\times\real_+^m$ be an optimal solution of \eqref{eq:cw_prox_prob} for a given $(x,y)\in \real\times\real^m$.
	Suppose that $v^* > 0$ (i.e., non-trivial),\footnote{We will call   an optimal solution  trivial if  $v^* = 0$.}  and denote $S = \{ j : w^*_j > 0 \}$.
	Then $|S|\leq \lambda^{-2}$.
\end{lemma}
\begin{proof}
	Since $(v^*, w^*)$ is an optimal solution of \eqref{eq:cw_prox_prob} and the objective function in \eqref{eq:cw_prox_prob} is twice continuously differentiable, $(v^*, w^*)$ satisfies the second order necessary optimality conditions \citep[Ex. 2.1.10]{B99}.
	That is, for any $d\in\real\times\real^m$ satisfying that $ (v^*, w^*)+d \in \real_+\times\real_+^m$ and $d^T \nabla h_\lambda(v^*, w^*;x,y) = 0$ it holds that
	$$ d^T \nabla^2 h_\lambda(v^*, w^*;x,y)d = d^T \left( \begin{array}{ccccc}
	1 & \lambda &\cdots  &\lambda  \\ 
	\lambda&  1&  0& 0 \\ 
	\vdots&  0&  \ddots&0  \\ 
	\lambda& 0  &  0& 1
	\end{array} \right) d \geq 0, $$
    where the first row/column corresponds to $v$ and the others correspond to
    $w$.  Noting that for any $j\in S$ it holds that $\frac{\partial
    h_\lambda}{\partial w_j} (v^*, w^*;x,y) = 0$, we have that  the  submatrix
    of $\nabla^2 h_\lambda(v^*, w^*;x,y)$ containing the rows and columns
    corresponding to the positive coordinates in $(v^*, w^*)$ must be positive
    semidefinite.

	Since the the minimal eigenvalue of this submatrix equals $ 1-\lambda\sqrt{|S|}$, we have that 
	$ \lambda^{-2} \geq |S|.$
\end{proof}
Moreover, the function $h_\lambda$ is monotonically decreasing in the sparsity level, which implies that instead of exhaustively checking the value of $h_\lambda$ for any sparsity level, we can employ a binary search.
Denote for any $s\in \{0, \ldots, m\}$ the $m+1$ possible solutions:
\begin{equation*}
\begin{aligned}
v^{(s)} &= \dfrac{1}{1-s \lambda^2}\left(|x| - \lambda\sum_{j=1}^s|y_j|\right) \\
w^{(s)}_j &= |y_j| - \lambda v^{(s)} \ \text{ for }  j\in [s], \ \text{ and } w^{(s)}_j = 0 \ \text{otherwise}.
\end{aligned}
\end{equation*}
\begin{lemma} \label{lem:decreasing-hb}
	Let $\bar{s} = \lfloor \lambda^{-2} \rfloor$.
	For all integer $s \in \{2,3,\ldots,\bar{s}  \}$, we have that 
	\begin{equation}
	\label{eq:decreasing-hb}
	h_\lambda(v^{(s)}, w^{(s)}; x, y) < h_\lambda(v^{(s-1)}, w^{(s-1)}; x, y).
	\end{equation}
\end{lemma}
 Lemma \ref{lem:decreasing-hb} follows from algebraic considerations, and thus its proof is deferred to Section \ref{app:prox_computation1}.
Its substantial implication is the following.
\begin{corollary}
	\label{cor:2}
	Suppose that there exists a non-trivial  optimal solution of \eqref{eq:cw_prox_prob}. 
	Denote $\bar{s} = \min(\lfloor \lambda^{-2}\rfloor, m) $  and let 
	\begin{equation*}
	\label{eq:opt-sparsity-single}
	s^* = \max \left\lbrace s \in \{0,\ldots, \bar{s}\}  :  v^{(s)}, w^{(s)}_s > 0\right\rbrace. 
	\end{equation*}
	Then $(v^{(s^*)}, w^{(s^*)} )$ is an optimal solution of \eqref{eq:cw_prox_prob}.
\end{corollary}

Note that since, by definition, the $s$ first entries of the vector $w^{(s)}$ are ordered in decreasing order, the constrained $w^{(s)}_s > 0$ ensures that the full vector $w^{(s)}$ has exactly $s$ nonzero entries, which are all strictly positive.

The final ingredient required for designing an efficient algorithm is the following monotone property of the feasibility criterion in problem~\eqref{eq:opt-sparsity-single}:
\begin{lemma}
	\label{lem:3}
For any $k \in [\bar{s}]$, we have
\[
v^{(k)} > 0, w^{(k)} > 0 \Rightarrow v^{(i)} > 0, w^{(i)} > 0, \ \ \forall i < k.
\]
\end{lemma}

This property, whose proof is also deferred to Section \ref{app:prox_computation1}, implies that the optimal sparsity parameter $s^*$ can be efficiently found using a binary search approach.

We conclude this section by combining all the ingredients above to develop Algorithm \ref{alg:prox_binary},  and to prove that it yields a solution to \eqref{eq:proximal}.
\begin{algorithm}
	\caption{Single-output robust-sparse proximal mapping}
	\label{alg:prox_binary}
	\textbf{Input:} $x \in \R$, $y \in \R^m$ sorted in decreasing magnitude order, $\lambda> 0$.
	\begin{algorithmic}[1]
		\State  $v^* = 0, w^* = |y|$
		\State $s_{\text{lb}} \leftarrow 0$, $s_{\text{ub}} \leftarrow \min(\lfloor \lambda^{-2} \rfloor, m)$, $s\leftarrow \lceil (s_{\text{lb}} + s_{\text{ub}})/2\rceil$
		\While{$s_{\text{lb}} \neq s_{\text{ub}}$}
		\State $v^{(s)} = \dfrac{1}{1-s \lambda^2}\left(|x| - \lambda\sum_{j=1}^s|y_j|\right)$
		\State $w^{(s)}_j = |y_j| - \lambda v^{(s)}$, $j \in [s]$ and $w^{(s)}_j = 0$ otherwise
		\If {$v > 0$, $w_s > 0$ } 
		\State $s_{\text{lb}} \leftarrow  s$, $s \leftarrow \lceil (s_{\text{lb}} +s_{\text{ub}} )/2 \rceil$
		\State $(v^* , w^*) \leftarrow (v , w)$
		\ElsIf{$v<0$} $s_{\text{ub}} \leftarrow  s$, $s \leftarrow  \lceil  (s_{\text{lb}} +s_{\text{ub}} )/2\rceil$
		\Else \ $s_{\text{lb}} \leftarrow  s$, $s \leftarrow \lceil (s_{\text{lb}} +s_{\text{ub}} )/2 \rceil$ 
		\EndIf
		\EndWhile
		\State \Return $(\sign(x) \cdot v^*, \sign(y) \circ w^*)$
	\end{algorithmic}
\end{algorithm}

\begin{theorem}[Prox computation]
	\label{thm:1}
	Let $(v^*_i, W^*_{i,:})$ be the output of Algorithm \ref{alg:prox_binary} with input $x_i, Y_{i,:}, \lambda$, assuming that each $Y_{i,:}$ is sorted in decreasing magnitude order.
	Then  $(v^*,W^*)$ is a solution to \eqref{eq:proximal}.
\end{theorem}
\begin{proof}
    For any $i=1,2,\ldots,n$, let $(v^*_i, W^*_{i,:})$ be the output of
    Algorithm \ref{alg:prox_binary} with input $x_i, Y_{i,:}, \lambda$.  We
    will show that   $(v^*,W^*)$ is an optimal solution to \eqref{eq:proximal}
    by arguing that Algorithm  \ref{alg:prox_binary} chooses the point with the
    smallest $h_\lambda$ value out of a feasible set of solutions containing an
    optimal solution of \eqref{eq:proximal}.

    For simplicity, and without loss of generality, let us consider the
    one-coordinate-one-row case, that is, $(v^*_i, W^*_{i,:}) \equiv (v^*,
    w^*)$, $(x_i, Y_{i,:}) \equiv (x, y)$; the proof for the general case is a
    trivial replication.
	
    By Lemma \ref{lem:5b} it is sufficient to prove that $(|v^*|, |w^*|)$ is an
    optimal solution of \eqref{eq:cw_prox_prob}, as this will imply the
    optimality of $(v^*, w^*)$; Recall that Lemma \ref{lem:12} establishes that
    there exists an optimal solution to  \eqref{eq:cw_prox_prob}.
	
    If the trivial solution is the only optimal solution to
    \eqref{eq:cw_prox_prob}, then obviously it will be the output of Algorithm
    \ref{alg:prox_binary}.  Otherwise, the  point described in  Corollary
    \ref{cor:2} is an optimal solution. Assume that Algorithm
    \ref{alg:prox_binary} returned the point
    $(v^{(s_{\text{out}})},w^{(s_{\text{out}})})$ for some $s_{\text{out}}\in
    [\bar{s}]$, meaning in particular that
    $(v^{(s_{\text{out}})},w^{(s_{\text{out}})})>0$.  By definition, $s^*\geq
    s_{\text{out}}$.  If $s_{\text{out}}< s^*$, then at some $s<s^*$ we had
    that $v^{(s)}<0$.  Since the value of  $v^{(i)} $  is monotonic decreasing
    in the sparsity level, this implies that $v^{(s^*)}<0$, which is a
    contradiction.
    
    Hence, if Algorithm \ref{alg:prox_binary} did not return
    the trivial solution, then $(v^*,w^*) = (v^{(s^*)},w^{(s^*)})$, meaning
    that $(\sign(x) \cdot v^*, \sign(y) \circ w^*)$ is a solution to
    \eqref{eq:proximal}.
\end{proof}
\paragraph{Time complexity of Algorithm~\ref{alg:prox_binary}} In the worst
case where $m \leq \lambda^{-2}$, the number of searches for finding $s^*$ is
at most $\log_2(m)$. Each search requires to compute $v^{(s)}$, and in
particular $\sum_{j=1}^s |y_j|$, as well as $w_j^{(s)}$, $j=1,\ldots,s$, each
taking $\mathcal{O}(s)$ steps. Thus, the overall loop complexity is
$\mathcal{O}(m)$.

Moreover, this algorithm assumes that the input vector $y$ is already sorted in
decreasing magnitude order. This can easily be achieved by a sorting procedure
in time $\mathcal{O}(m \log m)$.

\section{Multi-Output}
\label{sec:multi}

The efficient computation of the robust-sparse proximal mapping we derived for the single-output scenario will now be generalized to the multi-output case.
Although we use similar arguments and insights, the analysis is much more complicated and requires more delicate and advanced treatment.
Due to the tedious computations that accompany the analysis, the proofs are deferred to appendix~\ref{app:prox_computation2}.

When the network has multiple-output, the proximal operator $\text{prox}_{\lambda g}(X,Y)$ can be written as the solution set of 
\begin{equation}
\label{eq:multi-prox}
\begin{aligned}
\min_{V,W} \|V - X\|_F +   \|W - Y\|_F + 2\lambda \sum_{i=1}^n\sum_{j=1}^m\sum_{k=1}^p |W_{ij}V_{ik}|,
\end{aligned}
\end{equation}
where $V\in \R^{n\times p}$ and $W\in \R^{n\times m}$.
As in the single-output case, we observe that the proximal mapping \eqref{eq:multi-prox} is separable with respect to the $i$-th rows of the matrices $V$ and $W$, and that the signs of the decision variables are determined by the signs of $(X,Y)$. 
Therefore, it is enough to consider the problem related to the $i$-th row of $V$, denoted as $x$, and $i$-th row of $W$, denoted as $y$, i.e.,
\begin{equation}
\label{eq:5}
\begin{aligned}
\min_{v, w \in \R_+^p \times \R_+^m} h_\lambda(v,w;x,y), 
\end{aligned}
\end{equation}
where we redefine $h_\lambda(v,w;x,y)$ to include the multi-output case: $ h_\lambda(v,w;x,y) =  \frac{1}{2} \sum_{k=1}^p (v_k-|x_k|)^2 + \frac{1}{2} \sum_{j=1}^m (w_j - |y_j|)^2 
+ \lambda \sum_{k=1}^p v_k \sum_{j=1}^m w_j.$
To improve readability, we will abuse notation and just write $h_\lambda(v,w)$, assuming that $(x,y)$ are understood from context.

Using the same observations we exploited to enumerated all stationary points of the proximal mapping
in the single-output setup, we can identify the stationary points depending on the number of non zero elements of $v$ and $w$.
\begin{lemma}
	\label{lem:13}
	Let $(v^*, w^*)\in \real_+^p\times\real_+^m$ be an optimal solution of \eqref{eq:8} for a given $(x,y)\in \real\times\real^m$.
	Then
	\begin{enumerate}
		\item The vector $w^*$ satisfies that  for any $j,l\in [m]$ it holds that $w^*_j \geq w^*_l$ only if $|y_j| \geq |y_l|$.
		\item The vector $v^*$ satisfies that  for any $k,l\in [p]$ it holds that $v^*_k \geq v^*_l$ only if $|x_k| \geq |x_l|$.
		\item Let $\bar{x}, \bar{y}$ be the sorted vectors in descending magnitude order of $x$ and $y$ respectively. 
		Let $s_v = |\{k : v_k^* > 0\}|$ and $s_w = |\{j : w_j^* > 0\}|$. 
		If $v^*, w^* \neq 0$, then we have that for any $k \in \{k : v_k^* > 0\}$ and $j \in \{j : w_j^* > 0\}$, it holds that $v^* = v^{(s_v, s_w)}$ and $w^* = w^{(s_v, s_w)}$ where
	\end{enumerate}
		\begin{align}
		\label{eq:vopt-multi}
		v_k^{(s_v, s_w)} &= |x_k| +  \mu \left(\lambda^2 s_w \sum_{l=1}^{s_v} |\bar{x}_l| - \lambda\sum_{j=1}^{s_w} |\bar{y}_j|\right) \\
		\label{eq:wopt-multi}
		w_j^{(s_v, s_w)} &= |y_j| +  \mu \left(\lambda^2 s_v \sum_{l=1}^{s_w} |\bar{y}_l| - \lambda\sum_{k=1}^{s_v} |\bar{x}_k|\right)
		\end{align}
		and $\mu = (1 - s_v s_w \lambda^2)^{-1}$.
\end{lemma}

From the two first points in Lemma \ref{lem:13}, the argument in Remark~\ref{rem:sorting} is also valid in the multi-output case, and so we assume hereafter  that the input vectors $x, y$ are sorted in decreasing magnitude order.

Using the second order stationary conditions, we can generalize our sparsity bound in the single-output scenario, given in Lemma \ref{lem:4b}, to an upper bound on the product of the sparsities of the solutions based on the value of $\lambda$; indeed, $s_v = 1$ yields the bound in Lemma \ref{lem:4b}.
\begin{lemma}[Sparsity bound]
	\label{lem:7}
	Let $(v^*, w^*)\in \real_+^p\times\real_+^m$ be an optimal solution of \eqref{eq:5} for a given $(x,y)\in \real^p\times\real^m$.
	Denote $s_v = | \{ j : w^*_j > 0 \}|$ and $s_w = |\{j : w_j^* > 0\}|$.
	Then $s_v s_w \leq \lambda^{-2}$.
\end{lemma}
A possible algorithm for computing this proximal mapping would thus be to
compute the value of $h_\lambda \left(v^{(s_v, s_w)},w^{(s_v, s_w)}\right)$ for
each pair of sparsities $(s_v, s_w) \in \{0,\ldots,p\} \times \{0,\ldots,m\}$
satisfying $s_v s_w \leq \lambda^{-2}$ and return the pair achieving the
smallest value.

However, such an approach would be computationally
inefficient. In order to avoid computing the value of $h_\lambda$ at each
pair, we show the following monotonicity property of $h_\lambda$ in the
sparsity levels, which generalizes the same property in the single-output case.
\begin{lemma} \label{lem:8}
Given $(x,y)\in \real^p\times\real^m$, for all $s_v, s_w \in \{0,\ldots,p\} \times \{0,\ldots,m\}$ satisfying $s_v s_w < \lambda^{-2}$, we have 
\begin{align*}
h_\lambda(v^{(s_v, s_w)}, w^{(s_v, s_w)}) &< h_\lambda(v^{(s_v, s_w-1)}, w^{(s_v, s_w-1)}), \\
h_\lambda(v^{(s_v, s_w)}, w^{(s_v, s_w)}) &< h_\lambda(v^{(s_v-1, s_w)}, w^{(s_v-1, s_w)}).
\end{align*}
\end{lemma}

Moreover, the feasibility criterion $v \geq 0, w \geq 0$ also has a monotonic property:
\begin{lemma}\label{lem:4}
Let $(k,l) \in [p] \times [m]$ be such that $kl \leq \lambda^{-2}$.

If $v^{(k, l)} \geq 0$ and $w^{(k, l)} \geq 0$, then, $v^{(i, j)} \geq 0$ and $w^{(i, j)} \geq 0$ $\forall i = 1, \ldots, k$ and $\forall j = 1, \ldots, l$.
\end{lemma}
To properly address the complications arising from handling two intertwining sparsity levels at the same time, we introduce the notion of  \textit{maximal feasibility boundary (MFB)} which acts a frontier of possible sparsity levels.
\begin{definition}[Maximal feasibility boundary]
	\label{def:1}
We say that a sparsity pair $(s_v,s_w) \in \{0,\ldots,p\} \times \{0,\ldots,m\}$ is on the maximal feasibility boundary  (MFB) if incrementing either $s_v$ or $s_w$ results with a non-stationary point.
That is, if both of the following conditions hold:

$\bullet$ $v^{(s_v + 1, s_w)}_{s_v + 1} < 0$ or $w^{(s_v + 1, s_w)}_{s_w} < 0$ or $(s_v+1)s_w > \lambda^{-2}$,

$\bullet$ $v^{(s_v, s_w+1)}_{s_v} < 0$ or $w^{(s_v, s_w+1)}_{s_w+1} < 0$ or $s_v(s_w+1) > \lambda^{-2}$.
\end{definition}

The efficient computation of the multi-output robust-sparse proximal mapping is based on the fact that we only need to  compute the value of $h_\lambda$ for sparsity levels that are at the frontier of the MFB. 
This allows us to find the optimal sparsity in time $\mathcal{O}(p + m)$, improving upon the $\mathcal{O}(pm)$ complexity of the exhaustive search.
Algorithm \ref{alg:prox_binar_multi} implements the above by employing a binary search type procedure defined in Algorithm~\ref{alg:max-feasibility-boundary} to calculate the MFB.

\begin{algorithm}
	\caption{Multi-output robust-sparse proximal mapping}
	\label{alg:prox_binar_multi}
	\textbf{Input:} $x \in \R^p$, $y \in \R^m$ ordered in decreasing magnitude order, $\lambda> 0$.
	\begin{algorithmic}[1]
		\State Employ Algorithm~\ref{alg:max-feasibility-boundary}: Find the set of sparsity pairs $S = \{(s_v, s_w)\}$ that are on the MFB
		\State $h_{opt} \leftarrow \infty$
		\For{$(s_v, s_w) \in S$}
			\State Compute $v^{(s_v, s_w)}$ and $w^{(s_v, s_w)}$ as given in equations~\eqref{eq:vopt-multi}, ~\eqref{eq:wopt-multi}
			\If{$h_\lambda(v^{(s_v, s_w)},v^{(s_v, s_w)};|x|,|y|) < h_{opt}$}
				\State $h_{opt} = h_\lambda(v^{(s_v, s_w)},v^{(s_v, s_w)};|x|,|y|)$
				\State $v^* \leftarrow v^{(s_v, s_w)}$, $w^* \leftarrow w^{(s_v, s_w)}$
			\EndIf
		\EndFor
		\State \Return $(\sign(x) \circ v^*, \sign(y) \circ w^*)$
	\end{algorithmic}
\end{algorithm}

\begin{theorem}[Multi-output prox computation]
	\label{thm:1m}
	Let $(V^*_{:,i}, W^*_{i,:})$ be the output of Algorithm \ref{alg:prox_binar_multi} with input $X_{:,i}, Y_{i,:}, \lambda$, where each $X_{:,i}$, $Y_{i,:}$ are sorted in decreasing magnitude order. Then  $(V^*,W^*)$ is a solution to \eqref{eq:proximal}.
\end{theorem}

\paragraph{Time complexity of Algorithm~\ref{alg:prox_binar_multi}} It is easy to see that the maximal feasibility boundary contains at most $\min(m,p)$ pairs, and Algorithm~\ref{alg:max-feasibility-boundary} finds them all in time $\mathcal{O}(m+p)$. Then, for each such pair $(s_v, s_w)$, we must compute $v^{(s_v, s_w)}$ and $w^{(s_v, s_w)}$ and $h_\lambda(v^{(s_v, s_w)}, w^{(s_v, s_w)})$, which takes time $\mathcal{O}(m+p)$. The total complexity of Algorithm~\ref{alg:prox_binar_multi} is thus $\mathcal{O}(\min(m,p)(m+p))$. In most practical application, the output layer size $p$ can be considered $\mathcal{O}(1)$, so that the complexity of computing this proximal mapping is comparable to the complexity of computing one stochastic gradient.

\section{Related Work}
The path regularization approach to train neural networks can be traced back to
the seminal paper by \citet{Neyshabur2015}, who introduced the $p$-path-norm as
a heuristic proxy to control the \textit{capacity} of the network.

In this paper (cf. \autoref{prop:upper_bound}), we took a step forward by
moving from heuristic explanations to rigorous arguments by establishing  a new
connection between the $1$-path-norm and the Lipschitz constant of the network.
This result also reads as a relation between the 1-path-norm and the product
bound, of which variants have been found to be useful in deriving
generalization bounds \citep{Bartlett2017}.

Generalization bounds in terms of the $p$-path-norm were also derived in
\citep{Neyshabur2015}, but the question of how to methodologically exploit
these as regularizers remains open; our algorithmic contribution is a first
step in this direction. Additionally, issues regarding optimization with
path-norm regularization were reported by \citet{Ravi2019}, which examined a
conditional gradient method in the context of path-norm regularization.

A growing collection of works have focused on the task of network compression,
doing so via sparsity-inducing regularizers
\citep{Alvarez2016,Yoon2017,Scardapane2017,Lemhadri2019}. They have achieved a
great level of success by setting the regularization term in an \textit{ad-hoc}
manner. In contrast, we follow a principled regularization approach with
theoretical properties of generalization and robustness, and as a consequence,
we are able to quantify the sparsity of the resulting networks.

Moreover, the aforementioned works only use convex regularizers for which
efficient proximal mappings are available (see \citet{Bach2012} and references
therein). We tackle the much harder non-convex regularization task, and
derive a new method to compute the proximal mapping in this case. The merits
of non-convex non-smooth regularization, and difficulties regarding their
optimization, have been extensively studied in the imaging and signal
sciences, see e.g., \citep{Ochs2015} and the recent survey \citep{WCLQ18}.

Layer-wise constraints or regularization with matrix-norms, which are also
motivated by the product bound, have been used for robust
training \citep{cisse17a,Tsuzuku2018} and generative models \citep{Miyato2018}.
These focus on robustness with respect to the $\ell_2$-norm, which requires a
careful handling of operations on the singular values of the weight matrices,
and does not have the extra benefit of inducing sparsity.

In \autoref{sec:experiments} we compare to this class of methods for the
$\ell_\infty$-norm case, in which  a simple rescaling of the rows in the weight
matrices yields a numerically stable procedure \citep{Duchi2008,Condat2016}.

\section{Experiments}
\label{sec:experiments}
%
\begin{figure*}[t]
\centering
\begin{minipage}[c]{\textwidth}
\centering
     \includegraphics[width=0.95\textwidth]{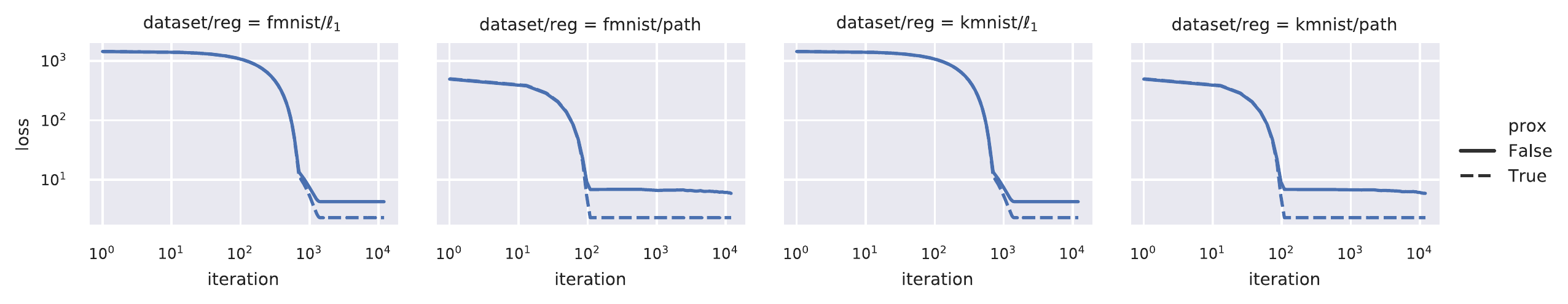}
     \caption{value of regularized cross-entropy loss across iterations.}
     \label{fig:loss_vs_iter}
     \includegraphics[width=0.95\textwidth]{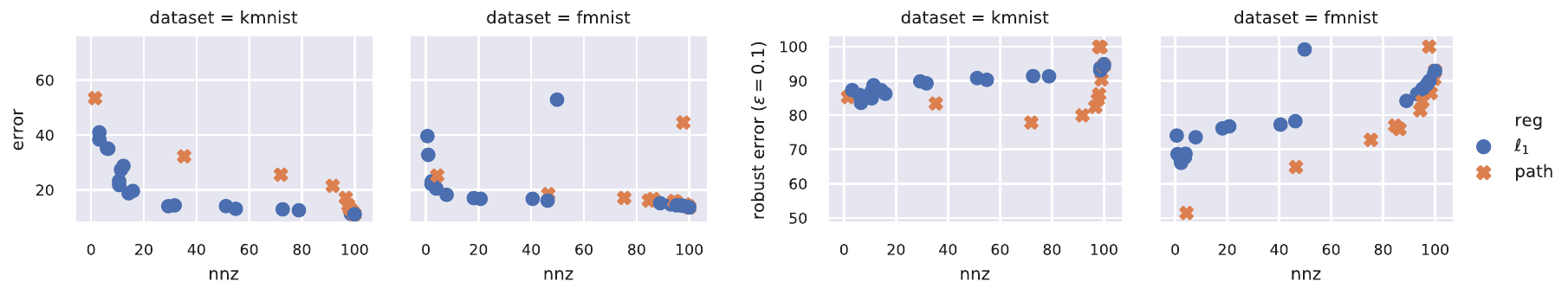}
     \caption{Misclassification test error (left) and robust test error (right) as a function of the percentage of nonzero weights.}
     \label{fig:nnz_vs_iter}
     \includegraphics[width=0.95\textwidth]{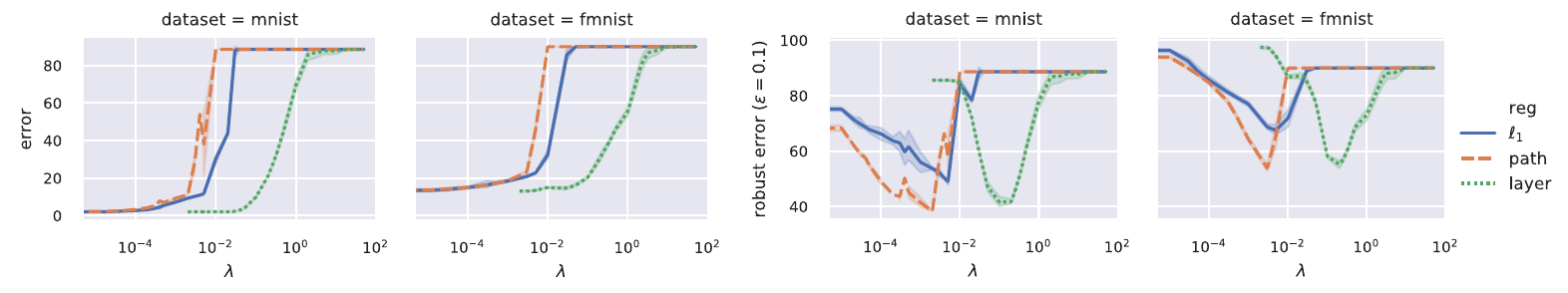}
     \caption{Misclassification test error (left) and robust test error (right) on the test set, as a function of the regularization parameter $\lambda$.}
     \label{fig:error_vs_robust}
\end{minipage}
\end{figure*}
We empirically evaluate shallow neural networks trained by regularized
empirical risk minimization \eqref{eq:objective_fn} using cross-entropy loss.
In terms of the weight matrices $V$ and $W$ of the network \eqref{eq:nn}, the following 
regularizers are considered:

\textbf{$\ell_1$ regularization. }We penalize the $\ell_1$-norm of the
parameters of the network, i.e., $g(V, W) = \| \vect(V)\|_1 + \|
\vect(W) \|_1$ in the objective function \eqref{eq:objective_fn}.

\textbf{1-path-norm regularization. } We set $g(V,
W)$ as $\sum_{i=1}^n\sum_{j=1}^m\sum_{k=1}^p |W_{ij}V_{ki}|$ in
the objective function \eqref{eq:objective_fn}.

\textbf{Layer-wise regularization (Parseval Networks). } we minimize the
cross-entropy loss with a hard constrain on the $\ell_\infty$-operator-norm of
the weight matrices i.e., $\|W\|_\infty \leq \lambda^{-1}$ and $\|V\|_\infty
\leq \lambda^{-1}$, as described by \citet{cisse17a}. The projection on such
set is achieved by projecting each row of the matrices onto an $\ell_1$-ball
using efficient algorithms \citep{Duchi2008,Condat2016}.

\textbf{Remark. } We will refer (incorrectly) to the training loop
defined by PyTorch's SGD optimizer as \textit{Stochastic gradient descent
(SGD)} (see the discussion in \autoref{sec:introduction}).

\textbf{Experimental setup. } Our benchmarks are the MNIST \citep{mnist},
Fashion-MNIST \citep{fashionmnist} and Kuzushiji-MNIST \citep{kmnist}. For a
wide range of learning rates, number of hidden neurons and regularization
parameters $\lambda$, we train networks with SGD and Proximal-SGD (with constant
learning rate). We do so for 20 epochs and with batch size set to 100. For each
combination of parameters we train 6 networks with the default random
initialization. Details and further experiments are reported in appendix
\ref{app:exp_details}.

\subsection{Convergence of SGD vs Proximal-SGD}
Due to the non-differentiability of the $\ell_1$- and path-norm regularizers,
we expect Proximal-SGD to converge faster, and to lower values of the
regularized loss, when compared to SGD.  This is examined in Figure
\ref{fig:loss_vs_iter}, where we plot the value of the loss function across
iterations.  For both SGD and Proximal-SGD, the loss function decays rapidly in
the first few epochs.  We then enter a second regime where SGD suffers from
slow convergence, whereas Proximal-SGD continues to reduce the loss at a fast
rate.  At the end of the 20-epochs, Proximal-SGD consistently achieves a lower
value of the loss.

An advantage of Proximal-SGD over plain SGD is that the proximal mappings of
both the $\ell_1$- and path-norm regularizers can set many weights to
\textit{exactly} zero. In \autoref{fig:nnz_vs_iter} we plot the average error
and robust test error obtained, as a fuction of the sparsity of the network.
Compared to $\ell_1$ regularization, the sparsity pattern induced by the
1-path-norm correlates with the robustness to a higher degree. As a drawback,
it appears that in more difficult datasets like KMNIST, the 1-path-norm
struggles to obtain good accuracy and sparsity simultaneously.

\subsection{The robustness-accuracy trade-off}
The relation between the Lipschitz constant of a network and its robustness to
adversarial perturbations has been extensively studied in the literature. In
\autoref{prop:upper_bound} we have shown that the 1-path-norm of a
single-output network is a tighter upper bound of its Lipschitz constant,
compared to the corresponding product bound.

To the best of our knowledge, the $\ell_1$-norm regularizer only provides an
upper boud on the already loose product bound \citep[Eq. (4)]{Neyshabur2015},
which makes it less attractive as a regularizer, despite its sparsity-inducing
properties. Hence, the 1-path-norm regularizer is, in theory, a better proxy
for robustness than the other regularization
schemes.

\autoref{fig:error_vs_robust} shows the misclassification error on clean and
adversarial examples as a function of  $\lambda$, and corresponds to the
learning rate minimizing the error on clean samples. The adversarial
perturbations were obtained by PGD \citep{Madry2018}.

Any training procedure which promotes robustness of a classifier may decrease
its accuracy, and this effect is consistently observed in practice
\citep{Tsipras2018}. Hence, the merits of a regularizer should be measured by
how efficiently it can trade-off accuracy for robustness. We observe that
for all three regularization schemes, there exists choices of $\lambda$ that
attain the best possible error on clean samples.

On the other hand, the error obtained by the $\ell_1$ regularization degrades
significantly. The layer-wise and 1-path-norm regularization achieve a
noticeably low error on adversarial examples. Comparing the latter schemes, the
1-path-norm regularization shows only a slight advantage over the layer-wise
methods, which merits further investigation.

\section{Future Work: Multilayer Extension}
A natural extension of our approach is to apply path regularization to
multi-layered networks.  Since the number of paths is potentially huge, this
scenario requires more sophisticated treatment, and hence left for future
research. Nonetheless, a trivial extension of our approach is to divide a
multi-layered network into pairs of consecutive layers, and  apply our method
in a sensible manner. We now describe this approach, to complement the theory.

Precisely, assume that the network contains an even number of layers. For some
lists of matrices $\mathbf{V}=(V^1,\ldots, V^k)$ and $\mathbf{W}=(W^1, \ldots,
W^k)$ of appropriate sizes, the network can be written as a composition of
activation functions and shallow networks
\begin{equation}
h_{\mathbf{V},
\mathbf{W}}:=h_{W^l, V^l} \circ \sigma \circ h_{W^{l-1}, V^{l-1}} \circ \ldots
\circ \sigma \circ h_{V^1, W^1}
\end{equation}

We build the regularized objective as
\begin{equation}
    \label{eq:objective_multilayer}
    \min_{\mathbf{V}, \mathbf{W}} \E_{(x, y)}[\ell(h_{\mathbf{V}, \mathbf{W}}(x), y)]
    + \lambda \sum_{l=1}^L P_1(V^l, W^l). 
\end{equation}
where we have introduced the shorthand
\begin{equation*}
    P_1(V^l, W^l):= \sum_{i=1}^{n_l} \sum_{j=1}^{m_l} \sum_{k=1}^{p_l} |W^l_{ij} V^l_{ik}|
\end{equation*}
For the 1-path-norm of the $l$-th subnetwork $h_{W^{l}, V^l}$.

Because the nonsmooth nonconvex regularizer in \eqref{eq:objective_multilayer}
is separable in the variables $\{(V_i, W_i): i=1,\ldots,k\}$, its proximal
mapping is indeed nothing but our multioutput prox algorithm (section
\ref{sec:multi}), applied independently to each shallow subnetwork component.
Thus, the proximal gradient methods can be applied efficiently to optimize
\eqref{eq:objective_multilayer}.

\section*{Acknowledgements}
This work is funded (in part) through a PhD fellowship of the Swiss Data
Science Center, a joint venture between EPFL and ETH Zurich. This project has
received funding from the European Research Council (ERC) under the European
Union's Horizon 2020 research and innovation programme (grant agreement n.
725594 - time-data). This work was supported by the Swiss National Science
Foundation (SNSF) under  grant number 200021\_178865 / 1.

\bibliography{references}
\bibliographystyle{icml2020}

%
%
%
\onecolumn
\appendix

\section{Proof of Theorem \ref{prop:upper_bound}}
\label{app:upper_bound}
We will first prove a particular case of \autoref{prop:upper_bound}, the single-output case ($p=1$).
\begin{proposition}
    \label{prop:particular_case}
Let $h_{V,W}(x)=V^T \sigma(Wx): \R^m \to \R$ be a neural network where $V \in
    \R^{n \times 1}$ and $W \in \R^{n \times m}$.  Suppose that that the derivative of the
    activation is globally bounded between zero and one. Its Lipschitz constant
    with respect to the $\ell_\infty$ norm (for the input space) and the
    $\ell_1$-norm (for the output space) is bounded as follows:
\begin{equation}
    \label{eq:upper_bound_app}
    L_{V,W} \leq \sum_{i=1}^n \sum_{j=1}^m |W_{i,j} V_{i,1}| \leq \|V\|_1 \|W\|_\infty
\end{equation}
\end{proposition}
First, note that because the output space is $\R$, the $\ell_1$-norm is just
the absolute value of the output. In this case the Lipschitz constant of the
single-output function $h$ is equal to the supremum of the $\ell_1$-norm of its
gradient, over its domain (c.f., \citet[Theorem 1]{Latorre2020}).
\begin{proof}
    \begin{align*}
        L_{V,W} &= \sup_x \| \nabla h_{V,W}(x) \|_1 \\
        &= \sup_x \sup_{\|t \|_\infty \leq 1} t^T \nabla h_{V,W}(x) \\
        &= \sup_x \sup_{\|t \|_\infty \leq 1} t^T W^T \sigma'(Wx)V \\
        &\leq  \sup_{0 \leq s \leq 1} \sup_{\|t\|_\infty \leq 1} t^T W^T \Diag(s) V \\
        &= \sup_{0 \leq s \leq 1} \sup_{\|t\|_\infty \leq 1} \sum_{i=1}^n \sum_{j=1}^m t_i (W^T \Diag(V))_{i,j} s_j \\
        &\leq \sum_{i=1}^n \sum_{j=1}^m \sup_{0 \leq s_j \leq 1} \sup_{-1 \leq t_i \leq 1} t_i (W^T \Diag(V))_{i,j}s_j \\
        &= \sum_{i=1}^n \sum_{j=1}^m | W^T \Diag(V) |_{i,j} = \sum_{i=1}^n \sum_{j=1}^m |W_{i, j} V_{i, 1}|
    \end{align*}
    This shows the first inequality in \eqref{eq:upper_bound_app}. We now show the second inequality. Denote the $i$-th row
    of the matrix $W$ as $w_i$:
    \begin{align*}
        \sum_{i=1}^n \sum_{j=1}^m |W_{i, j} V_{i, 1}|
        &= \sum_{i=1}^n |V_{i, 1}| \sum_{j=1}^m |W_{i,j}| \\
        &= \sum_{i=1}^n |V_{i, 1}| \|w_i\|_1 \\
        &\leq \sum_{i=1}^n |V_{i, 1}| \max_{j=1, \ldots, m} \|w_j\|_1 \\
        &=\sum_{i=1}^n |V_{i, 1}| \|W\|_\infty \\
        &= \|V\|_1 \|W\|_\infty
    \end{align*}
    In the fourth line we have used the fact that the $\ell_\infty$ operator
    norm of a matrix is equal to the maximum $\ell_1$-norm of the rows.

\end{proof}

\textbf{Proof of \autoref{prop:upper_bound}. } We now proceed with the general
case where $V \in \R^{n \times p}$, $W \in \R^{n \times m}$ and $h_{V,W}(x)=V^T
\sigma(Wx)$.

\begin{proof}
    Denote the columns of $V$, in order, as $V_1, \ldots, V_p$. Using
\autoref{prop:particular_case} we have
\begin{align*}
    \| V^T \sigma(Wx) - V^T \sigma(Wy)\|_1 &= \sum_{k=1}^p |V_k^T \sigma(Wx) - V_k^T \sigma(Wy)| \\
        & \leq \sum_{k=1}^p \sum_{i=1}^n \sum_{j=1}^m |W_{i, j} V_{i, k} | \|x - u \|_\infty \\
        & \leq \sum_{k=1}^p \|V_k\|_1 \|W \|_\infty \|x - y \|_\infty \\
        &= \|V^T\|_{\infty, 1} \|W \|_\infty  \|x - y \|_\infty
\end{align*}
    where in the fourth line we have used the fact that the $(\ell_\infty,
    \ell_1)$ operator norm of a matrix $V^T$ is equal to the sum of the
    $\ell_1$ norm of its rows i.e., the columns of $V$.  This shows that $L_{V,W} \leq
    \sum_{i=1}^n \sum_{j=1}^m \sum_{k=1}^p |W_{i, j} V_{i, k}| \leq
    \|V^T\|_{\infty, 1} \|W \|_\infty$
    
\end{proof}

\section{Proof of Theorem \ref{thm:2}}
\label{app:prox_grad}
In this section we prove the theoretical guarantees stated in Theorem \ref{thm:2} of the prox-grad method described by Algorithm \ref{alg:prox_grad}.
The first and second parts of  Theorem \ref{thm:2}  follow immediately from the results establish by \cite{BST13}.
Part two in Theorem \ref{thm:2}  states that Algorithm \ref{alg:prox_grad}  is globally convergent under the celebrated Kurdyka–Lojasiewicz (KL) property \cite{ABRS10}. 
The broad classes of semi-algebraic  and subanalytic functions, widely used in optimization, satisfy the KL property (see e.g. \citep[Section 5]{BST13}), and in particular,  most convex functions encountered in finite dimensional applications satisfy it (see \citep[Section 5.1]{BST13}). 
We refer the reader to the works \cite{ABRS10,ABF11,BST13}, in particular to \citep[Sections 3.2-3.5]{BST13} for additional information and results.

For Part three we require the sufficient decrease property stated next. 
\begin{lemma} [Sufficient decrease property {\citep[Lemma 2]{BST13}}]
 	\label{lem:sd}
 	Let $\Psi : \real^n\rightarrow\real$ be a continuously differentiable
 	function with gradient assumed $L_\Psi$-Lipschitz continuous, and let $\sigma :\real^n\rightarrow (-\infty,\infty] $ be a proper  l.s.c function satisfying that $\inf\sigma>-\infty$. 
 	Fix any $t \in (0, 1/L_\Psi)$. 
 	Then, for any $\bu\in\dom\sigma$ and any $\bu^+\in\real^n$ defined by
 	$$\bu^+\in\mathrm{prox}_{\sigma t} \left(\bu - t \nabla\Psi (\bu) \right) $$
 	we have
 	$$ \Psi (\bu) + \sigma(\bu)- \Psi (\bu^+) -\sigma(\bu) \geq  \frac{ 1 - t L_\Psi}{2 t} \|\bu^+ - \bu \|^2.$$
 \end{lemma}
\begin{proof} [Proof of Theorem \ref{thm:2}]
	The first and second parts follow from the results established by \cite{BST13}.
	We will now prove the third part.
	By Lemma \ref{lem:sd} we have that
	\begin{equation}
	\label{eq:4}
	\mathcal{F}(z^k)  - \mathcal{F} (z^{k+1}) = f (z^k) + \lambda g(z^k) - f (z^{k+1}) - \lambda g(z^{k+1})) \geq  \frac{1 - L \eta_k}{2 \eta_k} \|z^{k+1} - z^k \|^2.
	\end{equation}
	Hence $\{ f (z^k) + \lambda g(z^k) \}_{k\geq 0}$ is a non-increasing sequence that strictly decreasing unless a critical point is obtained in a finite number of steps.
	By summing \eqref{eq:4} over $k=0,1,\ldots,K$ and using the fact that $\{ f (z^k) + \lambda g(z^k) \}_{k\geq 0}$ is non-increasing and is bounded below by $\mathcal{F}_*$, we obtain that
	\begin{align*}
	\mathcal{F}(z^0)  - \mathcal{F}_* &\geq  \sum_{k=0}^{K}\frac{1 - L \eta_k}{2 \eta_k} \|z^{k+1} - z^k \|^2 \\
	&\geq \frac{1}{2}(c - L) K \min_{k=0,\dots,K} \|z^{k+1} - z^k \|_2^2.
	\end{align*}
	Consequently, 
	$$ \min_{k=0,\dots,K}  \|z^{k+1} - z^k \|_2 \leq \sqrt{\frac{2(\mathcal{F}(z^0)  - \mathcal{F}_*)}{(c-L)K}}.$$
\end{proof}

\section{Single output proximal map computation}
\label{app:prox_computation1}

This section provides the theoretical background and the required intermediate results to prove Theorem \ref{thm:1}.

%

\subsection{Moving to an Equivalent Easier Problem}
We are interested in minimizing the nonconvex twice continuously differentiable function
\begin{equation}
\label{eq:7}
\min_{v, w\in\real\times\real^m} \frac{1}{2} (v-x)^2 + \frac{1}{2} \sum_{j=1}^m (w_j - y_j)^2 +
\lambda |v| \sum_{j=1}^m |w_j|.
\end{equation}
The signs of the elements of the decision variables in (\ref{eq:7}) are determined by the signs of $(x, y)$, and consequently, the problem in  (\ref{eq:7})  is equivalent to  problem \eqref{eq:8}; this is (partly) formulated by Lemma \ref{lem:5}.
\begin{equation}
\label{eq:8}
\min_{v, w\in\real_+\times\real_+^m} h_\lambda(v,w;x,y) \equiv \frac{1}{2} (v-|x|)^2 + \frac{1}{2} \sum_{j=1}^m (w_j - |y_j|)^2 + \lambda v \sum_{j=1}^m w_j.
\end{equation}
\begin{lemma}
	\label{lem:5}
	Let $(v^*, w^*)\in\real_+\times\real_+^n$ be an optimal solution of problem \eqref{eq:8}.
	Then $(\sign(x) \cdot v^*, \sign(y) \circ w^*)$ is an optimal solution of problem \eqref{eq:7}.
\end{lemma}
\begin{proof}
	We have that
	\begin{align*}
	\tilde{h}_\lambda(v,w;x,y) &\equiv \frac{1}{2} (v-x)^2 + \frac{1}{2} \sum_{j=1}^m (w_j - y_j)^2 +
	\lambda |v| \sum_{j=1}^m |w_j|\\
	&=  \frac{1}{2} ( \sign(x) v-|x|)^2 + \frac{1}{2} \sum_{j=1}^m (\sign(y_j) w_j - |y_j|)^2 +
	\lambda |v| \sum_{j=1}^m |w_j|\\
	&\geq \frac{1}{2} (|v|-|x|)^2 + \frac{1}{2} \sum_{j=1}^m (|w_j| - |y_j|)^2 + \lambda v \sum_{j=1}^m w_j \\
	&\geq  h_\lambda(v^*,w^*;|x|,|y|),
	\end{align*}
	where the last inequality follows from the fact that $(v^*, w^*)$ is an optimal solution of \eqref{eq:8}.
	Since equality with the lower bound is attained by setting $(v,w) = (\sign(x) \cdot v^*, \sign(y) \circ w^*)$, we conclude that $(\sign(x) \cdot v^*, \sign(y) \circ w^*)$ is an optimal solution of \eqref{eq:7}.
\end{proof}
To summarize, we have established that, finding an optimal solution to \eqref{eq:8} and then changing signs accordingly, yields an optimal solution to \eqref{eq:7}.
We will now focus on obtaining an optimal solution for (\ref{eq:7}).
\subsection{Solving the Prox Problem}
First we note that  problem  \eqref{eq:8} is well-posed.
\begin{lemma} [Well-posedness of \eqref{eq:8}]
	\label{lem:1}
	For any $\lambda\geq 0$ and any $(x,y)\in\real\times\real^m$, the problem \eqref{eq:8} has a global optimal solution.
\end{lemma}
\begin{proof}
	The  claim follows from the  fact that the objective function is coercive, cf. \citep[Thm. 2.32]{B14}.
\end{proof}
In light of Lemma \ref{lem:1}, and due to the fact that in \eqref{eq:8} we minimize a continuously differentiable function over a closed convex set, the set of  optimal solutions of \eqref{eq:8} is a nonempty subset of the set of stationary points of \eqref{eq:8}.
These satisfy the following conditions (cf. \citep[Ch. 9.1]{B14}).
\begin{lemma} [Stationarity conditions]
	\label{lem:2}
	Let $(v^*, w^*)\in \real_+\times\real_+^m$ be an optimal solution of \eqref{eq:8} for a given $(x,y)\in \real\times\real^m$.
	Then 
	\begin{align*}
	w_j^* &= \max \left\lbrace 0, |y_j| - \lambda v^* \right\rbrace \ \text{ for any }  j=1,2,\ldots,m,\\
	v^* &= \max\left\lbrace 0, |x| -   \lambda \sum_{j=1}^{m} w^*_j \right\rbrace.
	\end{align*}
\end{lemma}
\begin{proof}
	The stationarity (first-order) conditions of \eqref{eq:8} (cf.  \citep[Ch. 9.1]{B14}) state that
	\begin{align*}
	\frac{\partial h_\lambda}{\partial v}(v^*,w^*;x,y) \begin{cases}
	= 0, & v^*>0,\\
	\geq 0, & v^* = 0,
	\end{cases} \quad 
	\frac{\partial h_\lambda}{\partial w_j}(v^*,w^*;x,y) \begin{cases}
	= 0, & w_j^*>0,\\
	\geq 0, & w_j^* = 0,
	\end{cases}
	\end{align*}
	which translates to
	\begin{align*}
	 v^* - |x| + \lambda \sum_{j=1}^m w^*_j \begin{cases}
	= 0, & v^*>0,\\
	\geq 0, & v^* = 0,
	\end{cases} \quad 
	w^*_j - |y_j| + \lambda v^* \begin{cases}
	= 0, & w_j^*>0,\\
	\geq 0, & w_j^* = 0,
	\end{cases}
	\end{align*}
	and the required follows. 
\end{proof}
The stationarity conditions given in Lemma \ref{lem:2} imply a solution form that we exploit in Algorithm \ref{alg:prox_binary}; this is described by Corollary \ref{cor:1}, where we use the convention that $\sum_{j=1}^{0} a_j \equiv 0$ for any $\{a_j\}\subseteq\real$.
\begin{corollary}
	\label{cor:1}
	Let $(v^*, w^*)\in \real_+\times\real_+^m$ be an optimal solution of \eqref{eq:8} for a given $(x,y)\in \real\times\real^m$.
	\begin{enumerate}
		\item The vector $w^*$ satisfies that  for any $j,l\in\{1,2,\ldots,m \}$ it holds that $w^*_j \geq w^*_l$ only if $|y_j| \geq |y_l|$.
		\item If $v^* = 0$, then $w^* = y$.
		\item If $v^* > 0$, and $s = | \{ j : w^*_j > 0 \} |$, then we have that 
		\begin{equation}
		\label{eq:2}
		v^* = \dfrac{1}{1-s \lambda^2}\left(|x| - \lambda\sum_{j=1}^s|\bar{y}_j|\right),
		\end{equation}
		where $\bar{y}$ is the sorted vector of $y$ in descending magnitude order. 
	\end{enumerate}
\end{corollary}
\begin{proof}
	The first part follows trivially from the stationarity conditions on $w^*$ given in Lemma \ref{lem:2}. The second part also follows trivially from the problem definition.
	
	From the first part and the conditions in Lemma \ref{lem:2} we have that $\sum_{j=1}^{m} w^*_j =  \sum_{j=1}^s|\bar{y}_j| - \lambda s v^*$.
	Plugging the latter to  the stationarity condition on $v^*$ (given in Lemma \ref{lem:2}) then implies the required.
\end{proof}
In our analysis, we strictly distinguish between the trivial solution $(v^*,w^*) = (0,y)$, and the non-trivial solution in which $v^*>0$.
A practical point-of-view suggests that if $v^* = 0$, then the corresponding succeeding weights should also be zero, even though the optimality conditions imply otherwise. 
However, to avoid hindering the training process, this observation is considered only in the end of the training.

Recall that our analysis so-far implies that the magnitude order of $y$ determines the order magnitude of $w$, effectively implying on set of non-zero entries in $w$ (cf. Remark \ref{rem:sorting}).
For clarity of indices, and without loss of generality, we assume throughout this section that the vector $y$ is already sorted in decreasing order of magnitude, that is $y \equiv \bar{y}$.
We will use, without confusion, both notation to describe the same entity in order to maintain coherence with our procedures and results. 

Denote
\begin{equation}
\label{eq:vsws}
\begin{aligned}
v^{(s)} &= \dfrac{1}{1-s \lambda^2}\left(|x| - \lambda\sum_{j=1}^s|y_j|\right) \\
w^{(s)}_j &= |y_j| - \lambda v^{(s)} \ \text{ for }  j=1,2,\ldots,s, \ \text{ and } w^{(s)}_j = 0 \ \text{otherwise}.
\end{aligned}
\end{equation}
Lemma \ref{lem:decreasing-hb} which states the monotonicity property
\begin{equation*}
h_\lambda(v^{(s)}, w^{(s)}; x, y) < h_\lambda(v^{(s-1)}, w^{(s-1)}; x, y).
\end{equation*} 
is proved next.
%
\begin{proof}[Proof of Lemma \ref{lem:decreasing-hb}]
	Recall that $ h_\lambda(v,w;x,y) := \frac{1}{2} (v-|x|)^2 + \frac{1}{2} \sum_{j=1}^m (w_j - |y_j|)^2 + \lambda v \sum_{j=1}^m w_j.$
	By plugging $w^{(s)}$ defined in \eqref{eq:vsws} to $h_\lambda$ we obtain that
	\begin{align*}
	h_\lambda(v^{(s)}, w^{(s)}; x, y) &= \frac{1}{2} (v^{(s)} - |x|)^2 + \frac{1}{2} \sum_{i=1}^s (|\bar{y}_i| - (|\bar{y}_i| - \lambda v^{(s)}))^2 + \frac{1}{2}\sum_{i=s+1}^m |\bar{y}_i|^2 + \lambda v^{(s)} \sum_{i=1}^s (|\bar{y}_i| - \lambda v^{(s)}) \\
	&= \frac{1}{2} (v^{(s)} - |x|)^2 + \frac{\lambda^2}{2} s (v^{(s)})^2 +  \frac{1}{2} \|y\|_2^2 - \frac{1}{2}\sum_{i=1}^s |\bar{y}_i|^2 + \lambda v^{(s)} \sum_{i=1}^s |\bar{y}_i| - \lambda^2 s (v^{(s)})^2 .
	\end{align*}
	Consequently,  plugging $v^{(s)}$, defined in \eqref{eq:vsws}, yields
	\begin{align*}
	h_\lambda(v^{(s)}, w^{(s)}; x, y) &= \frac{1}{2} \left(\frac{\lambda^2 s}{1 - \lambda^2 s} |x| - \frac{\lambda}{1-\lambda^2 s} \sum_{i=1}^s |\bar{y}_i| \right)^2  - \frac{\lambda^2 s}{2(1- \lambda^2 s)^2} \left( |x| - \lambda \sum_{i=1}^s |\bar{y}_i| \right)^2 \\
	&+ \frac{\lambda}{1 - \lambda^2 s}\sum_{i=1}^s |\bar{y}_i| \left(|x| - \lambda \sum_{i=1}^s |\bar{y}_i|  \right)  - \frac{1}{2} \sum_{i=1}^s |\bar{y}_i|^2 +  \frac{1}{2} \|y\|_2^2 \\
	&= \frac{\lambda^2 s}{2(1-\lambda^2 s)^2} x^2 (\lambda^2 s - 1) + \frac{\lambda^2}{2(1-\lambda^2 s)^2} \left(\sum_{i=1}^s |\bar{y}_i| \right)^2 (1 - \lambda^2 s - 2(1 - \lambda^2 s)) \\
	&+ |x| \sum_{i=1}^s |\bar{y}_i| \left(-\frac{\lambda^3 s}{(1-\lambda^2 s)^2} + \frac{\lambda^3 s}{(1-\lambda^2 s)^2} + \frac{\lambda}{1-\lambda^2 s} \right) - \frac{1}{2} \sum_{i=1}^s |\bar{y}_i|^2 +  \frac{1}{2} \|y\|_2^2 \\
	&= \frac{1}{2(1 - \lambda^2 s)} \left( -\lambda^2 s x^2 - \left(|x| - \lambda \sum_{i=1}^s |\bar{y}_i| \right)^2 + x^2 \right)- \frac{1}{2} \sum_{i=1}^s |\bar{y}_i|^2 +  \frac{1}{2} \|y\|_2^2 \\
	&= - \frac{1}{2(1-\lambda^2 s)}\left(|x| - \lambda \sum_{i=1}^s |\bar{y}_i| \right)^2 + \frac{1}{2} \|x\|_2^2 - \frac{1}{2} \sum_{i=1}^s |\bar{y}_i|^2 +  \frac{1}{2} \|y\|_2^2 \\
	&= - \left(1 + \frac{\lambda^2}{1 - \lambda^2 s}\right)\frac{1}{2(1-\lambda^2 (s-1))}\left(|x| - \lambda \sum_{i=1}^{s-1} |\bar{y}_i|  - \lambda |\bar{y}_s| \right)^2 + \frac{1}{2} \|x\|_2^2 - \frac{1}{2} \sum_{i=1}^{s-1} |\bar{y}_i|^2 - \frac{1}{2} |\bar{y}_s|^2 +  \frac{1}{2} \|y\|_2^2 \\
	&= h_\lambda(v^{(s-1)}, w^{(s-1)}; x, y) - \frac{1}{2(1 - \lambda^2 s + \lambda^2)}\left( -2\lambda |\bar{y}_s| \left( |x| - \lambda \sum_{i=1}^{s-1} |\bar{y}_i| \right) + \lambda^2 |\bar{y}_s|^2\right) \\
	&- \frac{\lambda^2}{2(1-\lambda^2 s)(1-\lambda^2 s + \lambda^2)}\left(|x| - \lambda \sum_{i=1}^s |\bar{y}_i| \right)^2 - \frac{1}{2} |\bar{y}_s|^2.
	\end{align*}
	Therefore,
	\begin{align*}
	&h_\lambda(v^{(s)}, w^{(s)}; x, y) - h_\lambda(v^{(s-1)}, w^{(s-1)}; x, y) \nonumber \\
	&=  - \frac{1}{2(1 - \lambda^2 s + \lambda^2)}\left( -2\lambda |\bar{y}_s| \left( |x| - \lambda \sum_{i=1}^s |\bar{y}_i| \right) - \lambda^2 |\bar{y}_s|^2 + \frac{\lambda^2}{1 - \lambda^2 s}\left( |x| - \lambda \sum_{i=1}^s |\bar{y}_i| \right)^2 + (1- \lambda^2 s + \lambda^2) |\bar{y}_s|^2\right) \nonumber  \\
	&= - \frac{1}{2(1 - \lambda^2 s + \lambda^2)}\left( (1- \lambda^2 s) |\bar{y}_s|^2 -2\lambda |\bar{y}_s| \left( |x| - \lambda \sum_{i=1}^s |\bar{y}_i| \right) + \frac{\lambda^2}{1 - \lambda^2 s}\left( |x| - \lambda \sum_{i=1}^s |\bar{y}_i| \right)^2\right) \nonumber \\
	&= - \frac{1 - \lambda^2 s}{2(1 - \lambda^2 s + \lambda^2)}\left( |\bar{y}_s|^2 -2\lambda |\bar{y}_s| v^{(s)} + \lambda^2(v^{(s)})^2\right) \nonumber \\
	&= - \frac{1 - \lambda^2 s}{2(1 - \lambda^2 s + \lambda^2)}\left( |\bar{y}_s| -\lambda v^{(s)} \right)^2 \leq 0, 
	\end{align*}
	meaning that
	\begin{equation*}
	h_\lambda (v^{(s)}, w^{(s)}; x, y) \leq h_\lambda(v^{(s-1)}, w^{(s-1)}; x, y).
	\end{equation*}
\end{proof}

We can now prove our key result formulated in Corollary \ref{cor:2}, that states that $(v^{(s^*)}, w^{(s^*)} )$ is an optimal solution of \eqref{eq:cw_prox_prob} for 
	\begin{equation*}
	s^* = \max \left\lbrace s \in [\bar{s}]  :  v^{(s)}, w^{(s)} > 0\right\rbrace, \ \text{ where } \ \bar{s} = \min(\lfloor \lambda^{-2}\rfloor, m).
	\end{equation*}
	
\begin{proof}[Proof of Corollary \ref{cor:2}]
	By Lemma \ref{lem:2b}, $(v^{(s^*)}, w^{(s^*)} )$  is a stationary point of \eqref{eq:cw_prox_prob}.
	Moreover, according to Corollary \ref{cor:1b} and Lemma \ref{lem:4b}, $(v^{(s^*)}, w^{(s^*)} )$ belongs to the set of $\bar{s}$ stationary points that are candidates to be optimal solutions of \eqref{eq:cw_prox_prob}.
	Invoking Lemma \ref{lem:decreasing-hb}, we have that 
	\begin{equation}
	\label{eq:52}
	h_\lambda(v^{(s^*)}, w^{(s^*)}; x, y) < h_\lambda(v^{(j)}, w^{(j)}; x, y), \quad \forall s^* > j.
	\end{equation}
	Hence, $(v^{(j)}, w^{(j)} )$  is not an optimal solution for any $j<s^* $.
	
	Let us now consider the complementary case.
	By Lemma \ref{lem:4b}, for any $i>\bar{s}$ the pair $(v^{(i)}, w^{(i)} )$  does not satisfy the second-order optimality conditions, and therefore is not an optimal solution.
	On the other hand, by the definition of $s^*$, for any $\bar{s} >i>s^*$ the pair $(v^{(i)}, w^{(i)} )$ is not a feasible solution , and subsequently not a stationary point. 
	To conclude, $h_\lambda(v^{(s^*)}, w^{(s^*)}; x, y) < h_\lambda(v^{(j)}, w^{(j)}; x, y)$ holds for any $j\neq s^*$ such that $(v^{(j)}, w^{(j)} )$ is a stationary point, meaning that $(v^{(s^*)}, w^{(s^*)} )$  is an optimal solution of \eqref{eq:cw_prox_prob}.
	
\end{proof}
	
	Finally, we will show that the problem of finding $s^*$ can be easily solved using binary search. To this end, we show that the feasibility criterion (i.e., $v^{(s)} > 0$ and $w^{(s)} > 0$) satisfies that 
	\[
	(v^{(k)}, w^{(k)}) \text{ is feasible } \Rightarrow (v^{(i)}, w^{(i)}) \text{ is feasible } \forall i < k
	\]
	
\begin{proof}[Proof of Lemma \ref{lem:3}]
	Suppose that $(v^{(k)}, w^{(k)})$ is feasible for some $k \in \{2, \ldots, \bar{s}\}$. By induction principle, it is sufficient to show that $(v^{(k-1)}, w^{(k-1)})$ is feasible in order to prove the result.
	
	By~\eqref{eq:vsws}, we have: 
	\[
	(1-k\lambda^2)v^{(k)} = |x| - \lambda \sum_{j=1}^k |y_j| = (1-k\lambda^2 + \lambda^2)v^{(k-1)} - |y_k|.
	\]
	which implies
	\[
	v^{(k-1)} = \frac{1}{(1-k\lambda^2 + \lambda^2)}((1-k\lambda^2)v^{(k)} + |y_k|) \geq 0.
	\]
	
	For $w^{(k)}$, it is easy to see from~\eqref{eq:vsws} that since the vector $y$ is sorted in decreasing order of magnitude, the vector $w^{(k)}$ is also sorted in decreasing order, and thus $w^{(k)}$ is feasible if and only if $w^{(k)}_k > 0$.
	\begin{align*}
	(1-k \lambda^2)w^{(k)}_k &= (1-k \lambda^2)|y_k| - \lambda |x| + \lambda^2 \sum_{j=1}^k |y_j| \\
	&= -\lambda |x| + (1 - (k-1)\lambda^2)|y_{k-1}| + \lambda^2 \sum_{j=1}^{k-1} |y_j| + \lambda^2 |y_k| + (1-k \lambda^2)|y_k| - (1 - (k-1)\lambda^2)|y_{k-1}| \\
	&= (1-(k-1) \lambda^2)w^{(k-1)}_{k-1} + (1-k\lambda^2 + \lambda^2)(|y_k| - |y_{k-1}|),
	\end{align*}
	where the last line uses the identity of the first line for $k-1$. We thus have:
	\[
	w^{(k-1)}_{k-1} = \frac{1}{(1-(k-1) \lambda^2)} (1-k \lambda^2)w^{(k)}_k + |y_{k-1}| - |y_k| \geq 0,
	\]
	since $|y_{k-1}| \geq |y_k|$ and $k \leq \lambda^{-2}$.
	
	Therefore, there exists a value $s^*$ such that $v^{(k)} > 0$ and $w^{(k)} > 0$ $\forall k \geq s^*$ and $v^{(k)} \geq 0$ or $w^{(k)} \geq 0$ $\forall k > s^*$. This value of $s^*$ can thus efficiently be found using binary search.
	
\end{proof}

\section{Multi-output proximal map computation}
\label{app:prox_computation2}

\subsection{Solving the prox problem}
Returning to the multi-output setting, recall that $h_{V,W}(x) = V^T \sigma(Wx)$ with $V \in \R^{p\times n}, W\in \R^{n\times m}$ and 
\[
g(V, W)= \sum_{i=1}^n\sum_{j=1}^m\sum_{k=1}^p W_{ij}V_{ki}.
\]
The proximal mapping can then be written as:
\begin{align*}
\text{prox}_{\lambda g}(\bar{V}, \bar{W}) &= \argmin_{V, W} \frac{1}{2} \|V - \bar{V}\|_F +  \frac{1}{2} \|W - \bar{W}\|_F + \lambda \sum_{i=1}^n\sum_{j=1}^m\sum_{k=1}^p W_{ij}V_{ki} \\
&= \argmin_{V, W} \sum_{i=1}^n \left( \frac{1}{2} \sum_{k=1}^p (V_{ki} - \bar{V}_{ki})^2 + \sum_{j=1}^p (W_{ij} - \bar{W}_{ij})^2 + \sum_{j=1}^m\sum_{k=1}^p W_{ij}V_{ki}\right).
\end{align*}
Noting that the proximal mapping is  separable with respect to the columns of $W$ and the rows of $V$, and using the same sign trick applied in the single-output case, it is enough to solve for any $i=1,\ldots, n$, 
\begin{equation}
\min_{v, w\in \R_+^p \times \R_+^m} h_\lambda(v,w;x,y) \equiv \frac{1}{2} \sum_{k=1}^p (v_k - |x_k|)^2 + \frac{1}{2} \sum_{j=1}^m (w_j - |y_j|)^2 + \lambda \sum_{j=1}^m\sum_{k=1}^p v_k w_j,
\label{eq:multi-ouptut-prox}
\end{equation}
where $x$ denotes the $i$-th row of $V$ and $y$ the $i$-th column of $W$, in order to compute the prox operator.

The stationarity conditions for \eqref{eq:multi-ouptut-prox} are stated next; the arguments are the same as in the single-output case.
\begin{lemma} [Stationarity conditions]
	\label{lem:6}
	Let $(v^*, w^*)\in \real_+^p\times\real_+^m$ be an optimal solution of \eqref{eq:multi-ouptut-prox} for a given $(x,y)\in \real^p\times\real^m$.
	Then 
	\begin{align*}
	w_j^* &= \max \left\lbrace 0, |y_j| - \lambda \sum_{k=1}^p v_k^* \right\rbrace \ \text{ for any }  j=1,2,\ldots,m,\\
	v_k^* &= \max\left\lbrace 0, |x_k| -   \lambda \sum_{j=1}^{m} w^*_j \right\rbrace \ \text{ for any }  k=1,2,\ldots, p.
	\end{align*}
\end{lemma}
The next lemma restates the result in Lemma \ref{lem:13} which expands on the monotonic relation in magnitude originally established for single-output networks in Corollary \ref{cor:1b}.
\begin{lemma}
	\label{lem:333}
	Let $(v^*, w^*)\in \real_+^p\times\real_+^m$ be an optimal solution of \eqref{eq:8} for a given $(x,y)\in \real^p \times\real^m$.
	\begin{enumerate}
		\item The vector $w^*$ satisfies that  for any $j,l\in\{1,2,\ldots,m \}$ it holds that $w^*_j \geq w^*_l$ only if $|y_j| \geq |y_l|$.
		\item The vector $v^*$ satisfies that  for any $k,l\in\{1,2,\ldots,p \}$ it holds that $v^*_k \geq v^*_l$ only if $|x_k| \geq |x_l|$.
		\item Let $\bar{x}, \bar{y}$ be the sorted vector of $x$ and $y$ respectively in descending magnitude order. Let $s_v = |\{k : v_k^* > 0\}|$ and  $s_w = |\{j : w_j^* > 0\}|$. 
		If $v^*, w^* \neq 0$, then
		\begin{align}
		v_k^* &= |x_k| +  \dfrac{1}{1 - s_v s_w \lambda^2}\left(\lambda^2 s_w \sum_{l=1}^{s_v} |\bar{x}_l| - \lambda\sum_{j=1}^{s_w} |\bar{y}_j|\right), \\
		w_j^* &= |y_j| +  \dfrac{1}{1 - s_v s_w \lambda^2}\left(\lambda^2 s_v \sum_{l=1}^{s_w} |\bar{y}_l| - \lambda\sum_{k=1}^{s_v} |\bar{x}_k|\right).
		\end{align}
	\end{enumerate}
\end{lemma}

\begin{proof}
The two first points are direct applications of the stationary conditions of Lemma~\ref{lem:6}.

From the conditions in Lemma \ref{lem:6} we have that 
\begin{align*}
\sum_{j=1}^{m} w^*_j &=  \sum_{j=1}^{s_w} |\bar{y}_j| - \lambda s_w \sum_{k=1}^p v_k^* \\
\sum_{k=1}^{p} v^*_k &= \sum_{k=1}^{s_v} |\bar{x}_k| -\lambda s_v \sum_{j=1}^m w_j^* \\
&= \sum_{k=1}^{s_v} |\bar{x}_k| - \lambda s_v  \sum_{j=1}^{s_w} |\bar{y}_j| + \lambda^2 s_v s_w \sum_{k=1}^p v_k^* \\
&= \frac{1}{1-\lambda^2 s_v s_w} \left( \sum_{k=1}^{s_v} |\bar{x}_k| - \lambda s_v  \sum_{j=1}^{s_w} |\bar{y}_j| \right).
\end{align*}
Thus,
\begin{align*}
\sum_{j=1}^{m} w^*_j &= \sum_{j=1}^{s_w} |\bar{y}_j| - \frac{\lambda s_w}{1-\lambda^2 s_v s_w} \left( \sum_{k=1}^{s_v} |\bar{x}_k| - \lambda s_v  \sum_{j=1}^{s_w} |\bar{y}_j| \right) \\
&= \frac{1}{1-\lambda^2 s_v s_w} \left( -\lambda s_w \sum_{k=1}^{s_v} |\bar{x}_k| + \sum_{j=1}^{s_w} |\bar{y}_j| \right).
\end{align*}

Plugging the latter to  the stationarity condition on $v^*$ (given in Lemma \ref{lem:6}) then implies the result.
\end{proof}

Finally, we show, as in the single-output case, that second order stationarity condition constraints the ranges of sparsities of $v^*$ and $w^*$; this relation is given by Lemma \ref{lem:7}, and is proved next.
\begin{proof} [Proof of Lemma \ref{lem:7}]
	Since $(v^*, w^*)$ is an optimal solution of \eqref{eq:multi-ouptut-prox} and the objective function in \eqref{eq:multi-ouptut-prox} is twice continuously differentiable, $(v^*, w^*)$ satisfies the second order necessary optimality conditions.
	That is, for any $d\in\real^p\times\real^m$ satisfying that $ (v^*, w^*)+d \in \real_+^p\times\real_+^m$ and $d^T \nabla h_\lambda(v^*, w^*;x,y) = 0$ it holds that
	$$ d^T \nabla^2 h_\lambda(v^*, w^*;x,y)d = d^T \left( \begin{array}{ccc}
	I_{p\times p} & \Lambda_{p \times m}   \\ 
	\Lambda_{m \times p }& I_{m\times m}
	\end{array} \right) d \geq 0, $$
	where the first row/column corresponds to $v$ and the others correspond to $w$, $I$ denotes the identity matrix and $\Lambda$ denotes a matrix completely filled with $\lambda$.
	Similarly as in the single output case, we require that the submatrix of $\nabla^2 h_\lambda(v^*, w^*;x,y)$ containing the rows and columns corresponding to the positive coordinates in $(v^*, w^*)$ is positive semidefinite. 
	Since the minimal eigenvalue of this submatrix equals $ 1-\lambda\sqrt{|S_v||S_w|}$, we have that 
	$$ \lambda^{-2} \geq |S_v||S_w|.$$	
\end{proof}

A possible way of solving this proximal problem is thus to exhaustively compute the value of $h_\lambda$ at each stationary point associated with sparsities $s_v = 1, \ldots, p$, $s_w = 1, \ldots, m$ such that $s_v s_w \leq \lambda^{-2}$.
However, trying all possible pairs of sparsities $(s_v, s_w)$ is computationally costly. Similarly as is the single output case, we can exploit some structure of the objective function $h_\lambda$ in order to reduce the possible candidate pairs of sparsities.

Without loss of generality, we assume hereafter  that the vectors $x, y$ are already sorted in decreasing order of magnitude.

Lemma \ref{lem:333} shows that for each pair $(s_v, s_w)$, $s_v = 0,1,\ldots,p$, $s_w = 0,1,\ldots,m$, there exists a stationary point $(v^{(s_v, s_w)}, w^{(s_v, s_w)})$ of $h_\lambda(\cdot, \cdot; x,y)$ such that $|\{k: v^{(s_v, s_w)}_k > 0\}| = s_v$, $|\{j: w^{(s_v, s_w)}_j > 0\}| = s_w$, given by
\begin{equation}
	\label{eq:vsws-multi}
	\begin{aligned}
	v^{(s_v, s_w)}_k &= |x_k| +  \dfrac{1}{1 - s_v s_w \lambda^2}\left(\lambda^2 s_w \sum_{l=1}^{s_v} |x_l| - \lambda\sum_{j=1}^{s_w} |y_j|\right) \ \text{ for }  k=1,2,\ldots,s_v, \ \text{ and } v^{(s_v, s_w)}_k = 0 \ \text{otherwise} \\
	w^{(s_v, s_w)}_j &= |y_j| +  \dfrac{1}{1 - s_v s_w \lambda^2}\left(\lambda^2 s_v \sum_{l=1}^{s_w} |y_l| - \lambda\sum_{k=1}^{s_v} |x_k|\right) \ \text{ for }  j=1,2,\ldots,s_w, \ \text{ and } w^{(s_v, s_w)}_j = 0 \ \text{otherwise}.
	\end{aligned}
\end{equation}

We now move to prove the monotonicity property stated in Lemma \ref{lem:8}.
\begin{proof} [Proof of Lemma \ref{lem:8}]
The proof follows exactly the same lines as in the single output case. We recall the definition of the objective function:
\[
h_\lambda(v,w;x,y) \equiv \frac{1}{2} \sum_{k=1}^p (v_k - |x_k|)^2 + \frac{1}{2} \sum_{j=1}^m (w_j - |y_j|)^2 + \lambda \left(\sum_{k=1}^p v_k\right) \left(\sum_{j=1}^m w_j\right).
\]

Plugging the definitions from equation~\eqref{eq:vsws-multi}, we have 
\begin{align}
h_\lambda &\left(v^{(s_v, s_w)}, w^{(s_v, s_w)}; x, y\right) =  \frac{s_v}{2}\left( \frac{1}{1 - \lambda^2 s_v s_w}\left(\lambda^2 s_w \sum_{k=1}^{s_v} |x_k| - \lambda \sum_{j=1}^{s_w} |y_j| \right)\right)^2 + \frac{1}{2} \sum_{k=s_v + 1}^p x_k^2 \nonumber \\
& + \frac{s_w}{2}\left( \frac{1}{1 - \lambda^2 s_v s_w}\left(\lambda^2 s_v \sum_{j=1}^{s_w} |y_j| - \lambda \sum_{k=1}^{s_v} |x_k|  \right)\right)^2 + \frac{1}{2}\sum_{j=s_w + 1}^m y_j^2 \nonumber \\
& + \frac{\lambda}{(1 - \lambda^2 s_v s_w)^2} \left(\sum_{k=1}^{s_v} |x_k| - \lambda s_v \sum_{j=1}^{s_w} |y_j|\right) \left( -\lambda s_w \sum_{k=1}^{s_v} |x_k| + \sum_{j=1}^{s_w} |y_j| \right) \nonumber \\
&= \frac{1}{2(1 - \lambda^2 s_v s_w)^2} \left( \left(\sum_{k=1}^{s_v} |x_k| \right)^2 (\lambda^4 s_v s_w^2 + \lambda^2 s_w - 2\lambda^2 s_w) + \left(\sum_{j=1}^{s_w} |y_j| \right)^2 (\lambda^2 s_v + \lambda^4 s_v^2 s_w  - 2\lambda^2 s_v) \right. \nonumber \\
& \left. \left(\sum_{k=1}^{s_v} |x_k| \right)\left(\sum_{j=1}^{s_w} |y_j| \right) (-2\lambda^3 s_v s_w -2\lambda^3 s_v s_w + 2\lambda + 2\lambda^3 s_v s_w)  \right) + \frac{1}{2}\sum_{k=s_v + 1}^p x_k^2 + \frac{1}{2}\sum_{j=s_w + 1}^m y_j^2 \nonumber \\
 \label{eq:25}
&= \frac{1}{2(1-\lambda^2 s_v s_w)} \left( -\lambda^2 s_w \left(\sum_{k=1}^{s_v} |x_k| \right)^2 -\lambda^2 s_v \left(\sum_{j=1}^{s_w} |y_j| \right)^2 + 2\lambda \left(\sum_{k=1}^{s_v} |x_k| \right)\left(\sum_{j=1}^{s_w} |y_j| \right) \right) + \frac{1}{2}\sum_{k=s_v + 1}^p x_k^2 + \frac{1}{2}\sum_{j=s_w + 1}^m y_j^2 \\
\label{eq:26}
&= \left(1 + \frac{\lambda^2 s_v}{1-\lambda^2 s_v s_w}\right) \frac{1}{2(1-\lambda^2 s_v (s_w-1))} \left( -\lambda^2 (s_w-1) \left(\sum_{k=1}^{s_v} |x_k| \right)^2 - \lambda^2 \left(\sum_{k=1}^{s_v} |x_k| \right)^2 \right. \nonumber \\
&\left. -\lambda^2 s_v \left( \left(\sum_{j=1}^{s_w-1} |y_j|\right)^2 + 2\lambda |y_{s_w}| \sum_{j=1}^{s_w-1} |y_j| + y_{s_w}^2 \right) + 2\lambda \sum_{k=1}^{s_v} |x_k| \left(\sum_{j=1}^{s_w-1} |y_j| + |y_{s_w}| \right)  \right) + \frac{1}{2}\sum_{k=s_v + 1}^p x_k^2 + \frac{1}{2}\sum_{j=s_w - 1+ 1}^m y_j^2 - \frac{1}{2}y_{s_w}^2. \\
\nonumber
\end{align}

By applying equation~\eqref{eq:25} at $s_v$, $s_w-1$, we can express the right hand side of equation~\eqref{eq:26} in terms of $h_\lambda \left(v^{(s_v, s_w-1)}, w^{(s_v, s_w-1)}; x, y\right)$ as:

\begin{align*}
h_\lambda &\left(v^{(s_v, s_w)}, w^{(s_v, s_w)}; x, y\right) = h_\lambda \left(v^{(s_v, s_w-1)}, w^{(s_v, s_w-1)}; x, y\right) + \frac{1}{2(1-\lambda^2s_v (s_w -1))} \left(-\lambda^2 \left(\sum_{k=1}^{s_v} |x_k| \right)^2 \right. \\
& \left. -\lambda^2 s_v |y_{s_w}|\left(2\sum_{j=1}^{s_w-1} |y_j| + |y_{s_w}|\right)  +2\lambda |y_{s_w}|\sum_{k=1}^{s_v} |x_k|  \right ) + \frac{\lambda^2 s_v}{2(1-\lambda^2s_vs_w)(1-\lambda^2 s_v(s_w-1))} \left( -\lambda^2 s_w \left(\sum_{k=1}^{s_v} |x_k| \right)^2 \right. \\
& \left. -\lambda^2 s_v \left(\sum_{j=1}^{s_w} |y_j| \right)^2 -2\lambda \left(\sum_{k=1}^{s_v} |x_k| \right)\left(\sum_{j=1}^{s_w} |y_j| \right) \right) -\frac{1}{2} y_{s_w}^2.
\end{align*}

Therefore:
\begin{align*}
h_\lambda &\left(v^{(s_v, s_w)}, w^{(s_v, s_w)}; x, y\right) - h_\lambda \left(v^{(s_v, s_w-1)}, w^{(s_v, s_w-1)}; x, y\right) \\
&= - \frac{1}{2(1-\lambda^2 s_v(s_w-1))}  \left( -2\lambda|y_{s_w}| \left(\sum_{k=1}^{s_v} |x_k| - \lambda s_v \sum_{j=1}^{s_w} |y_j| \right) -\lambda^2s_v |y_{s_w}|^2 + \lambda^2 \left(\sum_{k=1}^{s_v} |x_k|  \right)^2 \right. \\
& \left. +\frac{\lambda^2 s_v}{1-\lambda^2 s_v s_w} \left( \lambda^2 s_w \left(\sum_{k=1}^{s_v} |x_k| \right)^2 + \lambda^2 s_v \left(\sum_{j=1}^{s_w} |y_j| \right)^2 - 2\lambda \left(\sum_{k=1}^{s_v} |x_k| \right)\left(\sum_{j=1}^{s_w} |y_j| \right) \right) + (1-\lambda^2 s_v s_w + \lambda^2 s_v)|y_{s_w}|\right) \\
&= -\frac{1}{2(1-\lambda^2 s_v(s_w-1))} \left((1 - \lambda^2s_v s_w) y_{s_w}^2 - 2\lambda |y_{s_w}|(1-\lambda^2 s_v s_w) \sum_{k=1}^{s_v} v_k^{(s_v, s_w)} + \frac{\lambda^2}{1-\lambda^2 s_v s_w} \left(\sum_{k=1}^{s_v} |x_k| - \lambda s_v \sum_{j=1}^{s_w} |y_j| \right)^2 \right) \\
&= -\frac{1-\lambda^2 s_v s_w}{2(1-\lambda^2 s_v(s_w-1))} \left(|y_{s_w}| - \lambda \sum_{k=1}^{s_v} v_k^{(s_v, s_w)} \right)^2.
\end{align*}

The second result is obtain directly by symmetry between $v$ and $w$.
\end{proof}

In order to derive an efficient algorithm , we will again exploit the monotone property of the feasibility criterion $v^{(s_v, s_w)} > 0$, $w^{(s_v, s_w)} > 0$ restated from Lemma \ref{lem:4}:
\begin{lemma}[Restatement of Lemma \ref{lem:4}]
Let $(k,l) \in [p] \times [m]$ be such that $kl \leq \lambda^{-2}$.
Suppose that 
$$v^{(k, l)} \geq 0 \text{ and } w^{(k, l)} \geq 0.$$ 
Then for any $i = 1, \ldots, k$ and any $ j = 1, \ldots, l$, it holds that
$$v^{(i, j)} \geq 0 \text{ and } w^{(i, j)} \geq 0.$$
\end{lemma}
\begin{proof}[Proof of Lemma \ref{lem:4}]
Since the first $k$ entries of $v^{(k, l)}$ are ordered in decreasing order, we have that $v^{(k, l)} \geq 0$ if and only if $v^{(k, l)}_k \geq 0$. Similarly, $w^{(k, l)} \geq 0$ if and only if $w^{(k, l)}_l \geq 0$.

Suppose that $v^{(k, l)} \geq 0$ and $w^{(k, l)} \geq 0$. By induction, in order to prove the result, it is sufficient to prove that $v^{(k-1, l)}_{k-1} \geq 0$, $v^{(k, l-1)}_k \geq 0$, $w^{(k-1, l)}_{l} \geq 0$ and $w^{(k, l-1)}_{l-1} \geq 0$. 
We only prove the result for $v$, as the proof for $w$ is identical.

Using equation~\eqref{eq:vsws-multi}, we have:
\begin{align}
\label{eq:feas-monotone-v}
(1 - kl\lambda^2) v_k^{(k,l)} &= (1-kl\lambda^2) |x_k| + \lambda^2 l \sum_{i=1}^k |x_i| - \lambda \sum_{j=1}^l |y_j| \\
&= (1-kl\lambda^2) |x_k| + (1-(k-1)l\lambda^2) |x_{k-1}| - (1-(k-1)l\lambda^2) |x_{k-1}| + \lambda^2 l\sum_{i=1}^{k-1} |x_i| + \lambda^2 l |x_k| \lambda \sum_{j=1}^l |y_j| \nonumber\\
&= (1 - (k-1)l\lambda^2) v_{k-1}^{(k-1,l)} + (1 - (k-1)l\lambda^2)(|x_k| - |x_{k-1}|) \nonumber.
\end{align}

Therefore:
\begin{equation*}
v_{k-1}^{(k-1,l)} = \frac{1 - (k-1)l\lambda^2}{1 - kl\lambda^2}v_k^{(k,l)} + |x_{k-1}| - |x_k| \geq 0,
\end{equation*}
since the vector $x$ is ordered in decreasing order of magnitude, and thus $|x_{k-1}| - |x_k| \geq 0$.

Using again equation~\eqref{eq:feas-monotone-v}, we have:
\begin{align*}
(1 - kl\lambda^2) v_k^{(k,l)} &= (1 - kl\lambda^2) |x_k| + (1 - k(l-1)\lambda^2) |x_k| - (1 - k(l-1)\lambda^2) |x_k| \\
& + \lambda^2(l-1)\sum_{i=1}^k |x_i| + \lambda^2 \sum_{i=1}^k |x_i| - \lambda \sum_{j=1}^{l-1} |y_j| - \lambda |y_l| \\
&= (1 - k(l-1)\lambda^2) v_{k}^{(k,l-1)} - k\lambda^2 |x_k| + \lambda^2 \sum_{i=1}^k |x_i| -\lambda |y_l|, 
\end{align*}
where the last equality follows (again) from  equation ~\eqref{eq:feas-monotone-v} for $v_{k}^{(k,l-1)}$.
Thus,
\begin{equation}
\label{eq:29}
(1 - k(l-1)\lambda^2) v_{k}^{(k,l-1)} = (1 - kl\lambda^2) v_k^{(k,l)} + k\lambda^2 |x_k| - \lambda^2 \sum_{i=1}^k |x_i| + \lambda |y_l|.
\end{equation}

From the definition of $v_k^{(k,l)}$ (equation~\eqref{eq:vsws-multi}), we have that $v_k^{(k,l)} \geq 0$ is equivalent to the condition:
\[
|x_k| \geq \frac{\lambda \sum_{j=1}^l |y_j| - l\lambda^2 \sum_{i=1}^k |x_i| }{1 - kl \lambda^2}.
\]
Plugging this inequality in equation~\eqref{eq:29}, we obtain:

\begin{align}
(1 - k(l-1)\lambda^2) v_{k}^{(k,l-1)} &\geq (1 - kl\lambda^2) v_k^{(k,l)} + \frac{k\lambda^2}{1 - kl \lambda^2} \left(\lambda \sum_{j=1}^l |y_j| - l \lambda^2 \sum_{i=1}^k |x_i| \right) + \lambda |y_l| - \lambda^2 \sum_{i=1}^k |x_i| \nonumber \\
&= (1 - kl\lambda^2) v_k^{(k,l)} + \frac{\lambda}{1 - kl \lambda^2} \left(k\lambda^2 \sum_{j=1}^l |y_j| - kl \lambda^3 \sum_{i=1}^k |x_i| + (1 - kl \lambda^2) |y_l| - \lambda (1 - kl \lambda^2) \sum_{i=1}^k |x_i| \right) \nonumber \\
\label{eq:29a}
&= (1 - kl\lambda^2) v_k^{(k,l)} + \frac{\lambda}{1 - kl \lambda^2} \left(k\lambda^2 \sum_{j=1}^l |y_j| + (1 - kl \lambda^2) |y_l| - \lambda \sum_{i=1}^k |x_i| \right).
\end{align}

From the definition of $w_l^{(k,l)}$ (equation~\eqref{eq:vsws-multi}), we have that $w_l^{(k,l)} \geq 0$ is equivalent to the condition:
\begin{equation}
\label{eq:30}
(1 - kl \lambda^2) |y_l| + k\lambda^2 \sum_{j=1}^l |y_j| - \lambda \sum_{i=1}^k |x_i| \geq 0.
\end{equation}

Since the expression of equation~\eqref{eq:30} is exactly the same as the one inside the parentheses of equation~\eqref{eq:29a}, plugging this relation to~\eqref{eq:29} thus shows that $(1 - k(l-1)\lambda^2) v_{k}^{(k,l-1)} \geq 0$, i.e. $v_{k}^{(k,l-1)} \geq 0$.
\end{proof}

We now introduce the efficient procedure to compute the maximal feasibility boundary (MFB), and prove that it indeed delivers, as promised, all sparsity pairs in the MFB set.

\addtocounter{algorithm}{1}
\begin{algorithm}
	\caption{Finding sparsity pairs on the maximal feasibility boundary}
	\label{alg:max-feasibility-boundary}
	\textbf{Input:} $x \in \R^p$, $y \in \R^m$ ordered in decreasing magnitude order, $\lambda> 0$.
	\begin{algorithmic}[1]
	\State $s_v \leftarrow 0$, $s_w \leftarrow m$
	\State $S \leftarrow \emptyset$
	\State $maximal \leftarrow True$
	\While{$s_v \leq p$ and $s_w \geq 0$}
		\State Compute $v_{s_v}^{(s_v, s_w)}$ and $w_{s_w}^{(s_v, s_w)}$ as shown in equation~\eqref{eq:vsws-multi}
		\If{$v_{s_v}^{(s_v, s_w)} < 0$ or $w_{s_w}^{(s_v, s_w)} < 0$ or $s_v s_w \geq \lambda^{-2}$}
			\If{$maximal$} 
				\State $S \leftarrow S \cup  \{(s_v-1, s_w)\}$
				\State $maximal \leftarrow False$
			\EndIf
			\State $s_w \leftarrow s_w-1$
		\Else 
			\State $s_v \leftarrow s_v+1$ 
			\State $maximal \leftarrow True$
		\EndIf
	\EndWhile
	\If{$s_v == p+1$}
		\State $S \leftarrow S \cup \{(s_v-1, s_w)\}$
	\EndIf
	\State \Return S
	\end{algorithmic}
\end{algorithm}
	 
\begin{lemma}
The set $S$ returned by Algorithm~\ref{alg:max-feasibility-boundary} contains all, and only, the sparsity pairs that are on the maximal feasibility boundary.
\end{lemma}

\begin{proof}
	First recall that the MFB is defined as all pairs $(s_v,s_w) \in \{0,\ldots,p\} \times \{0,\ldots,m\}$ satisfying the conditions:
\begin{enumerate}
	\item $v^{(s_v, s_w)}_{s_v} > 0$ and $w^{(s_v, s_w)}_{s_w} > 0$ and $s_v s_w \leq \lambda^{-2}$,
	
	\item $v^{(s_v + 1, s_w)}_{s_v + 1} \leq 0$ or $w^{(s_v + 1, s_w)}_{s_w} \leq 0$ or $(s_v+1)s_w > \lambda^{-2}$ or $s_v = p$,
	
	\item $v^{(s_v, s_w+1)}_{s_v} \leq 0$ or $w^{(s_v, s_w+1)}_{s_w+1} \leq 0$ or $s_v(s_w+1) > \lambda^{-2}$ or $s_w = m$.
\end{enumerate}

Algorithm \ref{alg:max-feasibility-boundary} plays on the properties of  \textit{feasibility}-\textit{infeasibility} of the sparsity levels to build the MFB.
We say that a pair of the sparsity levels of $v$ and $w$ $(s_v, s_w)$ is \textit{feasible} if  $v_{s_v}^{(s_v, s_w)} \geq 0$, $w_{s_w}^{(s_v, s_w)} \geq 0$ and $s_v s_w < \lambda^{-2}$, and denote this by the property $P(i, j)$, i.e. 
$$(i,j) \text{ is feasible } \Leftrightarrow  P(i, j).$$ 

Our claim can be read as:  Let $(i, j)\in \{0,\ldots,p\} \times \{0,\ldots,m\}$, then $(i, j)$ is added to $S$ by Algorithm~\ref{alg:max-feasibility-boundary} if and only if $(i, j)$ belongs to the MFB, i.e.,
\[
(i, j) \in \text{MFB} \Leftrightarrow (i, j) \in S.
\]
Obviously, only feasible sparsity pairs belong to the MFB, and it is quite easy to see that only feasible sparsity pairs will belong to an output $S$ of Algorithm~\ref{alg:max-feasibility-boundary}.
Indeed, Algorithm~\ref{alg:max-feasibility-boundary} monotonically decrements $s_w$ starting from $s_w = m$ and increments $s_v$ starting from $s_v = 0$. For each value of $s_w$, it increases $s_v$ while the current pair $(s_v, s_w)$ is feasible (lines $12-15$). Once it reaches an infeasible point $(i,s_w)$, and in the case where $s_v$ has been increased at least once for this particular value of $s_w$, it adds to $S$ the pair encountered just before, i.e., $(i-1, s_w)$, and then decrements $s_w$ (lines $6-11$).

We first prove the $\Rightarrow$ statement. Suppose that some pair $(i, j)$ belongs to the MFB. Let us first leave aside the corner cases, and assume that $i < p$ and $j < m$.

Suppose first that $s_w$ reaches $j$ before $s_v$ reaches $i$, i.e., $s_v < i$. Since the pair $(i,j)$ is feasible, and due to the monotonicity property of the feasibility condition (Lemma~\ref{lem:8}), all pairs $(k,s_w)$ with $k \leq i$ must be feasible. Therefore, $s_v$ will be increased until reaching $i+1$. By definition of the MFB, the pair $(i+1, j)$ must be infeasible. Since $s_v$ has necessarily been increased at least once for this value of $s_w=j$, and so the pair $(i+1-1, j) = (i,j)$ will be added to $S$ before decrementing $s_w$.

In the special case where $i = p$, no infeasible point will be found. The loop will thus finish with $s_w = j$ and $s_v = p+1$. The condition at line $17$ will thus hold, and the pair $(p,j)$ will be added to $S$.

Suppose now that $s_v$ reaches $i$ before $s_w$ reaches $j$, i.e., $s_w > j$. Since $(i,j)$ is in the MFB, then the pair $(i,j+1)$ must be infeasible. Thanks to the monotonicity property of the feasibility condition (Lemma~\ref{lem:8}), all pairs $(s_v,k)$ with $k \geq i$ must also be infeasible. Therefore, $s_w$ will be decreased until reaching $s_w = j$. Then, similarly as in the previous case, since $(i,j)$ is feasible, $s_v$ will be increased, and the pair $(i,j)$ added to $S$.

We now prove the $\Leftarrow$ statement. We show that if $(i,j)$ is added to $S$, then it must belong to the MFB, i.e., it satisfies all three properties recalled in the beginning of the proof.

Let us first show that for each pair $(s_v, s_w)$ encountered during the algorithm, the pair $(s_v-1, s_w)$ is always feasible (or $s_v = 0$). 
We can show that this property is conserved each time the algorithm either increases $s_v$ or decreases $s_w$. First note that the pair $(0, m)$ is always feasible. The algorithm will then necessarily first goes to the pair $(1, m)$ and $P(1,m)$ is true. Then suppose that $P(s_v, s_w)$ is true for some pair $(s_v, s_w)$ encountered during the algorithm. Then, if $s_v$ is increases, it means that the pair $(s_v, s_w)$ is feasible. The next encountered pair is then $(s_v+1, s_w)$ and $P(s_v+1, s_w)$ is true. On the other hand, suppose that $s_w$ is decreased. The next encountered pair is thus $(s_v, s_w-1)$. Since $P(s_v, s_w)$ is true, it means that $(s_v-1, s_w)$ is feasible. By Lemma~\ref{lem:8}, it implies that $(s_v-1, s_w-1)$ is also feasible, and thus $P(s_v, s_w-1)$ is true. We thus proved that $P(s_v, s_w)$ is true for any pair $(s_v, s_w)$ encountered during the algorithm. Therefore, since any pair added to $S$ is of the form $(s_v-1, s_w)$ for some pair $(s_v, s_w)$ encountered during the algorithm, then any pair added to $S$ must be feasible.

The second property of the MFB is straightforward to show. Indeed, if $(i-1, j)$ is added to $S$, it means that the pair $(i,j)$ is infeasible due to condition on line $6$. 

Finally, the third property follows from the fact that, when reaching $s_w=j$, $s_v$ must be increased at least once for adding a pair of the form $(i,j)$ to $S$. Let $s_v^{(j)}$ be the value of $s_v$ when the algorithm reaches $s_w = j$. We necessarily have $s_v^{(j)} \leq i$. This implies that the pair $(s_v^{(j)}, j+1)$ is infeasible, otherwise $s_v$ would have been increased to a greater value at the previous value $s_w = j+1$. By Lemma~\ref{lem:8}, and since $s_v^{(j)} \leq i$ this implies that the pair $(i,j+1)$ is also infeasible, hence the result.

\end{proof}

\paragraph{Time complexity of Algorithm~\ref{alg:max-feasibility-boundary}} At each iteration of the loop, either $s_v$ is incremented by $1$ or $s_w$ is decremented by $1$. Since $s_v$ starts from $0$ and $s_w$ from $m$, and that the stopping criterion is $s_v > p$ or $s_w < 0$, it follows that the maximal number of iterations inside the loop is $m+p$. At each iteration, we must compute $v_{s_v}^{(s_v, s_w)}$ and $w_{s_w}^{(s_v, s_w)}$, which requires in particular to compute $\sum_{k=1}^{s_v} |x_k|$ and $\sum_{j=1}^{s_w} |y_j|$. However, these cumulative sums can be efficiently computed before the loop in time $\mathcal{O}(m+p)$, so that computing $v_{s_v}^{(s_v, s_w)}$ and $w_{s_w}^{(s_v, s_w)}$ inside the loop can be done in constant time. The overall complexity of this algorithm is thus $\mathcal{O}(m+p)$.

Moreover, we can see that each time we add a pair to $S$, we must both decrement $s_w$ by $1$ (just after adding the element in the algorithm), and increment $s_v$ by $1$ (in order for the boolean $maximal$ to become true again). Therefore, there can be at most $\min(m,p)$ pairs in the final set $s$ at the end of the algorithm.

Merging all previous results, we can finally prove Theorem~\ref{thm:1m}.

\begin{proof}[Proof of Theorem~\ref{thm:1m}]
Thanks to the separability argument, it is sufficient to prove that Algorithm~\ref{alg:prox_binar_multi} returns a solution of problem~\eqref{eq:5}.

Lemma~\ref{lem:13} states that given the number of nonzero elements $s_v = |\{k:v_k^* > 0\}|$, $s_w = |\{j:w_j^* > 0\}|$, the optimal solution $(v^*, w^*)$ can be obtained in close form (equations~\eqref{eq:vopt-multi},~\eqref{eq:wopt-multi}).

Due the monotonicity property of the objective function $h_\lambda$ (Lemma~\ref{lem:8}), it follows that the sparsity pair $(s_v, s_w)$ of the optimal solution must lie on the MFB. Indeed, if it does not lie on the MFB, then it means that the candidate solution associated with either the sparsity pair $(s_v+1, s_w)$ or $(s_v, s_w+1)$ must be feasible. According to Lemma~\ref{lem:8}, this pair would then yield a lower value of $h_\lambda$, and would then be a better solution.

Algorithm~\ref{alg:prox_binar_multi} computes the candidate solution associated with all sparsity pair lying on the MFB, and returns the one achieving the lowest value of $h_\lambda$. Therefore, the returned solution must necessarily be the optimal solution.
\end{proof}

\section{Experimental details and other plots}
\label{app:exp_details}
We consider the following values for the parameters that determine the training loop:
\begin{itemize}
    \item batch size: \texttt{100}
    \item epochs: \texttt{20}
    \item learning rate: \texttt{1e-1, 1e-2, 1e-3, 1e-4, 5e-1, 5e-2, 5e-3, 5e-4}
    \item dataset: \texttt{mnist, fmnist, kmnist}
    \item hidden neurons: \texttt{200}
    \item lambda ($\lambda$): \texttt{0., 1e-5, 1e-4, 1e-3, 1e-2, 1e-1, 1e0, 1e1, 1e2,
               2e-5, 2e-4, 2e-3, 2e-2, 2e-1, 2e0, 2e1, 2e2,
               3e-5, 3e-4, 3e-3, 3e-2, 3e-1, 3e0, 3e1, 3e2,
               4e-5, 4e-4, 4e-3, 4e-2, 4e-1, 4e0, 4e1, 4e2,
               5e-5, 5e-4, 5e-3, 5e-2, 5e-1, 5e0, 5e1, 5e2}
\end{itemize}

The $\ell_\infty$-bounded adversarial examples used to evaluate the robustness
of the networks were generated using the PGD method described in
\citep{Madry2018} and implemented in the \textit{advertorch} toolbox
(\url{https://github.com/BorealisAI/advertorch}) using the following
parameters:
\begin{itemize}
    \item epsilon: \texttt{0.05, 0.1, 0.15, 0.2, 0.25, 0.3}
    \item iterations: 40
    \item step size: epsilon / 20
    \item random initialization: \texttt{True}
\end{itemize}

\subsection{sparsity per iteration}
One advantage of the proximal mapping of the 1-path-norm and the $\ell_1$-norm
is that they can set many weights to exactly zero. This has the effect of
providing sparse networks from early iterations. This is in contrast to SGD
with a constant stepsize which does not generate sparse iterates. In Figures 4, 5, 6 and 7
we plot the percentage of nonzero weights as a
function of the iteration count, for both plain SGD and proximal SGD. We observe that in fact
this is the case, and that the sparsity of the $\ell_1$ and 1-path-norm regularized network
can be controlled with the regularization parameter $\lambda$.
\begin{figure*}
    \centering
    \label{fig:sparse_vs_iter_path_fmnist}
    \includegraphics[width=0.95\textwidth]{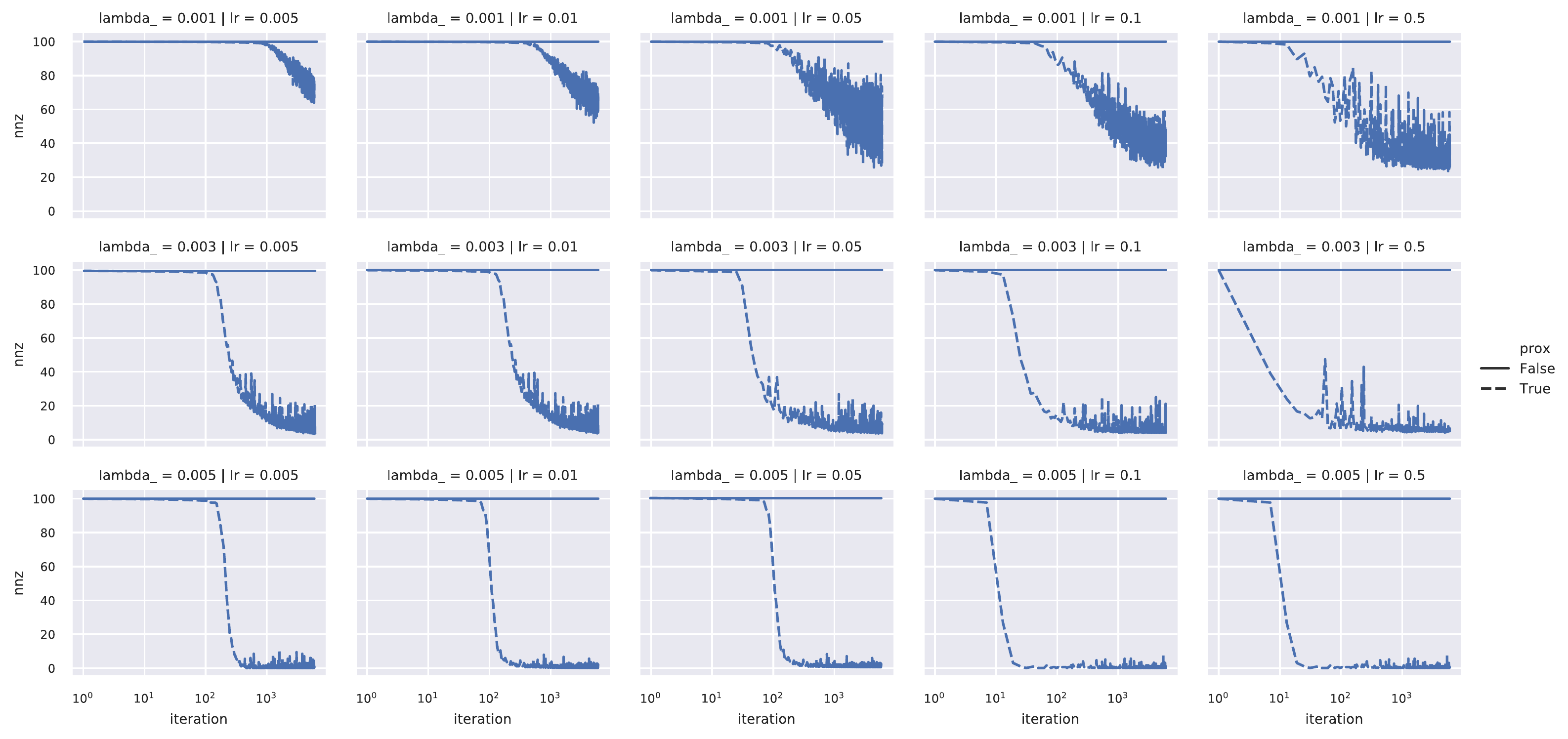}
    \caption{percentage of nonzero weights in the network, as a function of iteration count (path regularization - fmnist dataset).}
\end{figure*}
\begin{figure*}
    \centering
    \label{fig:sparse_vs_iter_path_kmnist}
    \includegraphics[width=0.95\textwidth]{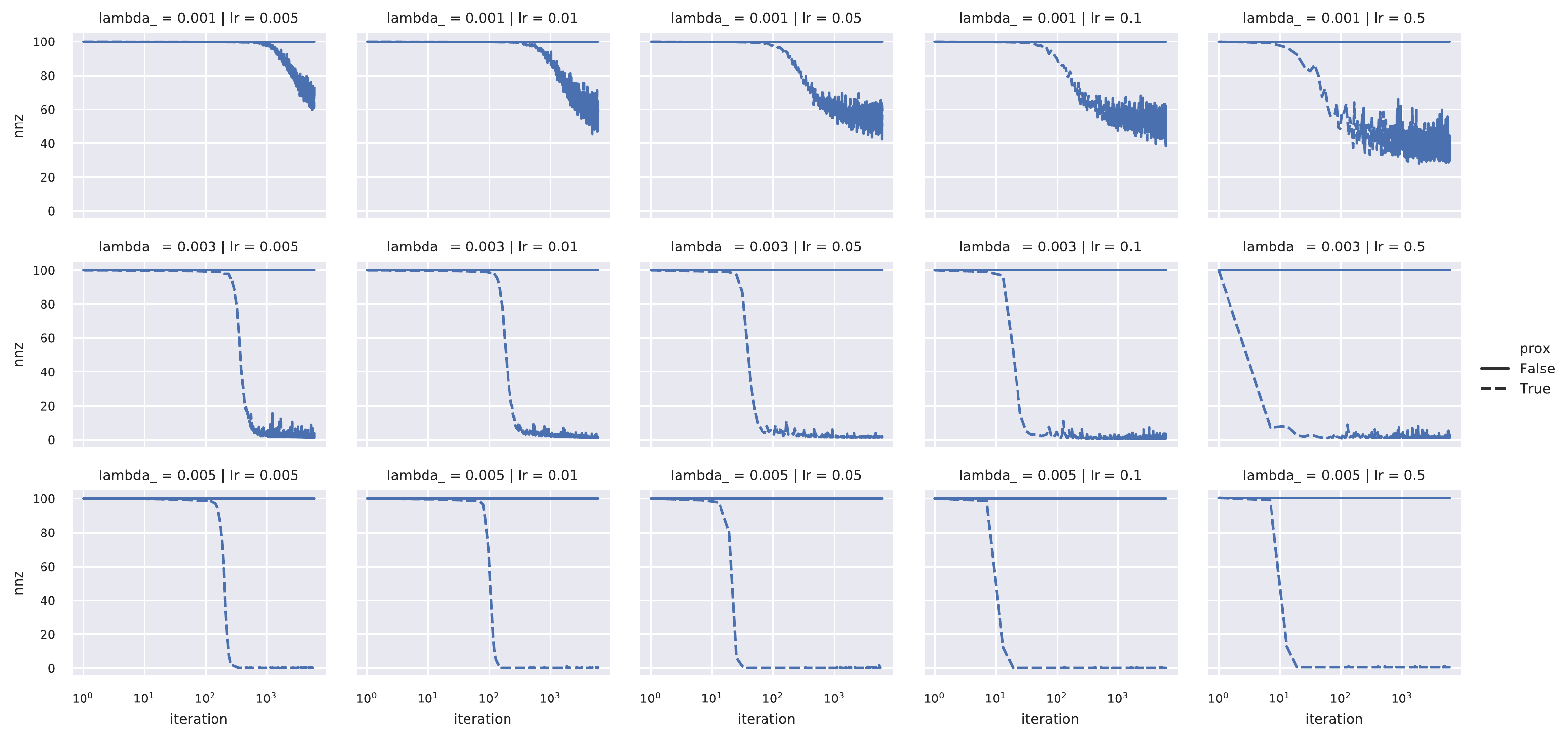}
    \caption{percentage of nonzero weights in the network, as a function of iteration count (path regularization - kmnist dataset).}
\end{figure*}
\begin{figure*}
    \centering
    \label{fig:sparse_vs_iter_l1_fmnist}
    \includegraphics[width=0.95\textwidth]{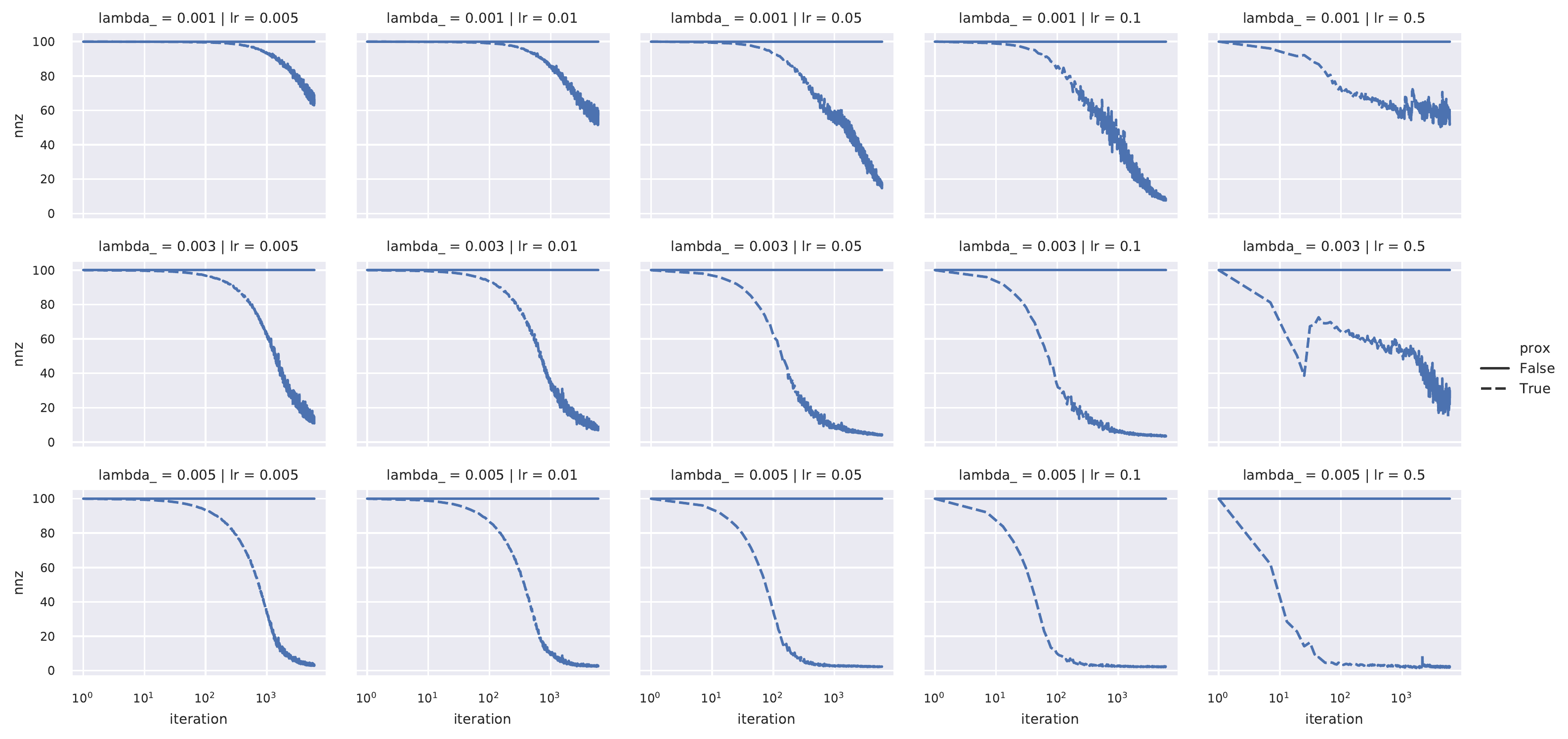}
    \caption{percentage of nonzero weights in the network, as a function of iteration count ($\ell_1$ regularization - fmnist dataset).}
\end{figure*}
\begin{figure*}
    \centering
    \label{fig:sparse_vs_iter_l1_kmnist}
    \includegraphics[width=0.95\textwidth]{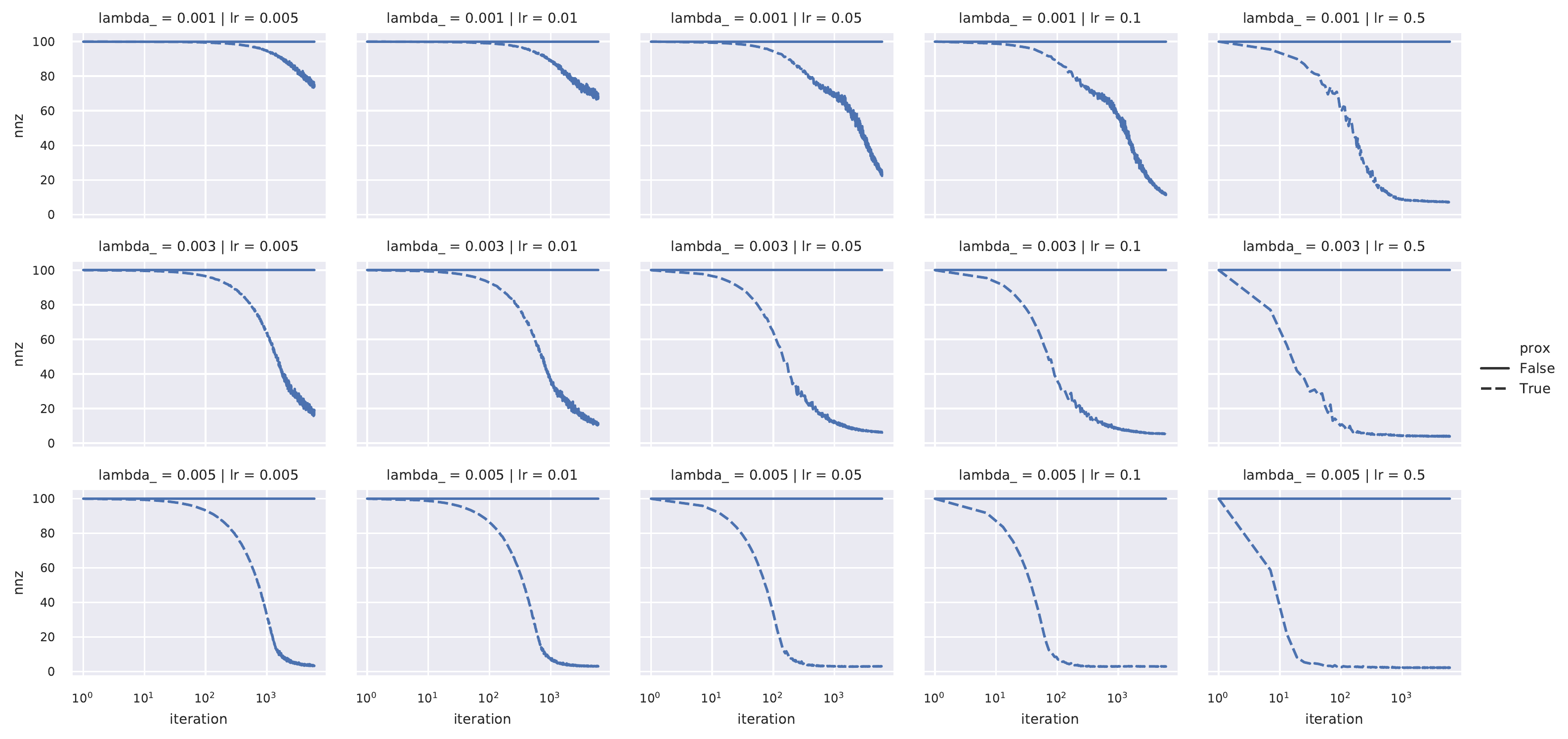}
    \caption{percentage of nonzero weights in the network, as a function of iteration count ($\ell_1$ regularization - kmnist dataset).}
\end{figure*}

\subsection{Robustness vs accuracy tradeoff}
For all possible values of $\lambda$, in Figure 8 we plot
the data corresponding to the lerning rate with least error. We plot the value
of the error on clean samples and the error on adversarial examples. This
allows us to understand the tradeoff between accuracy and robustness that is
controlled by the regularization paramter $\lambda$.
\begin{figure*}
    \centering
    \label{fig:error_vs_robust}
    \includegraphics[width=0.95\textwidth]{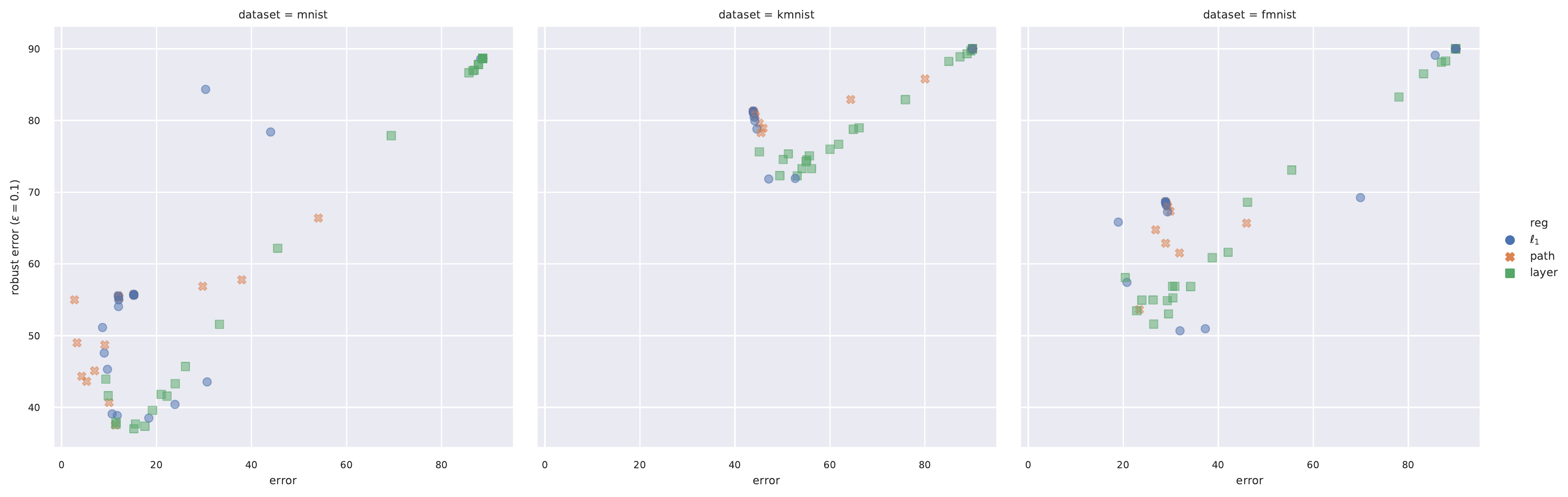}
    \caption{Robustness vs accuracy tradeoff for the different regularizers studied.}
\end{figure*}


\end{document}